\documentclass{article} 
\usepackage{nips15submit_e,times}
\usepackage{hyperref}
\usepackage{url}
\usepackage{graphicx}
\usepackage{float}
\usepackage{subfig}
\usepackage[outercaption]{sidecap}
\usepackage{amssymb}
\usepackage{amsmath}
\usepackage{amsthm}
\usepackage{algorithm}
\usepackage{algorithmic}
\usepackage{cite}

\title{A Sufficient Statistics Construction of Bayesian Nonparametric Exponential Family Conjugate Models}

\author{
Robert Finn \\
Department of Computer Science and Engineering\\
The Ohio State University\\
Columbus, OH\\
\texttt{finn@cse.ohio-state.edu} \\
\And
Brian Kulis\\
Department of Electrical and Computer Engineering \\
Boston University \\
Boston, MA\\ 
\texttt{bkulis@bu.edu} \\
}

\nipsfinalcopy 

\begin{document}

\maketitle

\begin{abstract}

Conjugate pairs of distributions over infinite dimensional spaces are prominent in statistical learning theory, particularly due to the widespread adoption of Bayesian nonparametric methodologies for a host of models and applications.  Much of the existing literature in the learning community focuses on processes possessing some form of computationally tractable conjugacy as is the case for the beta  and gamma processes (and, via normalization, the Dirichlet process). For these processes, proofs of conjugacy and requisite derivation of explicit computational formulae for posterior density parameters are idiosyncratic to the stochastic process in question. As such, Bayesian Nonparametric models are currently available for a limited number of conjugate pairs, e.g. the Dirichlet-multinomial, beta-Bernoulli, beta-negative binomial, and gamma-Poisson process pairs. It is worth noting that in each of these above cases, the likelihood process belongs to the class of \emph{discrete }exponential family distributions. The exclusion of \emph{continuous} likelihood distributions from the known cases of Bayesian Nonparametric Conjugate models stands as a glaring disparity in the researcher's toolbox.

Our goal in this paper is twofold. We first address the problem of obtaining a general construction of prior distributions over infinite dimensional spaces possessing distributional properties amenable to conjugacy.  Our result is achieved by generalizing Hjort's construction of the beta process via appropriate utilization of sufficient statistics for exponential families. Second, we bridge the divide between the discrete and continuous likelihoods by way of illustrating a \emph{canonical construction} for triples of stochastic processes whose L\'{e}vy measure densities are from positive valued exponential families, and subsequently demonstrate that these triples in fact form the prior, likelihood, and posterior in a conjugate family when viewed in the setting of conditional expectation operators. Our canonical construction subsumes known computational formulae for posterior density parameters in the cases where the likelihood is from a discrete distribution belonging to an exponential family.

\end{abstract}

\section{Introduction}

Since \cite{Raiffa} first formalized the notion of conjugate prior families in 1961, they have repeatedly earned their distinguished role as the key to the operability of Bayesian  modeling. Indeed, in both the parametric and nonparametric cases, the ability to perform statistical inference in a computationally efficient manner hinges on the existence of such conjugate prior families. 

Although conjugacy is key in both parametric and nonparametric modeling venues, the manners in which conjugacy assumes its role in these two arenas currently exist in unsettled contrast with another. To be sure, in the parametric setting, conjugate families such as the normal gamma, multinomial Dirichlet, and the Bernoulli beta, are familiar and often employed in the task of inference. Furthermore, such families afford the researcher convenient formulae for the updating of posterior parameters as a function of the prior parameters and observed data. Thus, the acceptance and application of conjugate families in the parametric, i.e. finite dimensional case, is received with a unified and uncontroversial disposition within the machine learning community. 

Unfortunately, at the present, in an infinite dimensional setting, no such intuitive methodology exists to derive simplistic computational formulae for the updating of posterior parameters. Consequently, no unified theory exists for the treatment of conjugate families in the case of Bayesian nonparametric models. This state of affairs is quite unfortunate as Bayesian nonparametric (BNP) modeling is a prominent and widely used technique in the machine learning community, providing a broad class of statistical models which are more flexible than classical nonparametric models and more robust than both classical and Bayesian parametric models. BNP models such as the Chinese restaurant process, the Indian buffet process, and Dirichlet process mixture models have obtained great success in problem domains such as clustering, dictionary learning, and density estimation, respectively \cite{Aldous,Ferguson,Griffiths,Kulis}. This success is in large part due to the adaptive nature of BNP models. As such, this modeling framework permits the data to determine the level of model complexity rather than  entail the specification of the complexity level by the researcher.

 In order for the adaptive framework of BNP to yield a computationally tractable model, one is commonly required to construct a pair of conjugate distributions defined over an infinite dimensional space,a process which, to date obtains no general theory or procedure. Given that the class of BNP models greatly enrich the researcher's toolkit, we feel they are deserving of the effort required to achieve a greater understanding of how to unify the notion of conjugacy and subsequent proof apparatuses thereof in an infinite dimensional setting. 

This lack of readily perceptible and intuitive updating formulae of posterior parameters conveys an enigmatic character to the proof techniques applied in the derivation of such updating formulae. While this objectionable state of affairs has inspired a number of researchers to produce intricate and creative proofs of the validity of updating formulae, the apparent asymmetry of complexity in the parametric and nonparametric cases serves as an obstruction to the free usage of a wider class of nonparametric models. This apparent asymmetry will be shown, in fact, to be illusory. We will prove in the present work that updating formulae for a wide class of nonparametric, i.e. infinite dimensional, exponential family models are in fact derivable \emph{directly} from the updating formulae for the analogous finite dimensional models under mild regularity conditions, and that this illusion may be dispelled in the cases of \emph{both} discrete \emph{and} continuous likelihoods. The general applicability of the theorems herein stand in contrast to previous results and methodologies of their derivation regarding conjugacy of singularly specific pairs of stochastic processes. 

Previously, success in constructing such conjugate pairs of stochastic processes in an infinite dimensional setting has been achieved in specific cases producing, for example, the Dirichlet, gamma, and beta processes priors. In each of these cases, the construction of a suitable likelihood/prior pair and subsequent illustration of the desired form for the posterior yielding conjugacy, is tailored to the conjugate pair in question. For example, in the case of the Dirichlet process, conjugacy of the multinomial and Dirichlet processes arises directly from the conjugacy of their marginals. In contrast, conjugacy in the case of the beta and Bernoulli processes is defined and obtained in a seemingly less direct manner, i.e. via the form of the densities for the processes' respective L\'{e}vy measures. Although this modification of the definition of conjugacy which employs the density of the associated L\'{e}vy measures, deviates from the standard definition in the finite dimensional, i.e. parametric, setting, it does yield a tractable definition which directly relates to a specification of the distributional properties of the \emph{infinitesimal} increments of the stochastic processes under consideration.

Construction of families which are conjugate in the modified sense requisite for a nonparametric setting, entails both the construction of a prior with distributional qualities appropriate to the modeling task at hand \emph{and} the ability to derive a conjugate posterior from this form of the prior and the chosen likelihood. With respect to the latter of the two constructions, i.e. the existence of a conjugate posterior, a number of distribution specific techniques have been successful, e.g. in the cases of the gamma and beta processes. Such constructions succeed by way of the derivation of explicit formulae for the updating of posterior parameters as a function of the prior parameters and observed data. Specifically, this has been achieved in cases such as the Dirichlet-multinomial, gamma-Poisson, beta-Bernoulli, and beta-negative binomial processes \cite{Lin,Hjort,Thibaux,Broderick}. In these cases, although conjugacy is addressed via the modified definition employing the density of the associated L\'{e}vy process, the actual proof techniques employed in establishing this form of conjugacy are conceptually orthogonal to one another in each case and require lengthy, involved, and dense arguments.  A limited number of more general construction techniques either implicitly lurk in the theory of conditional measures over infinite dimensional spaces, or have explicitly been formulated to prove the existence of a conjugate prior, albeit with varying success. For example, recently \cite{Orbanz} has obtained a mathematical framework  for proving the existence of a conjugate posterior over an infinite dimensional space \emph{given} one has in hand an appropriate prior over that space. Unfortunately, he requires the stochastic processes under consideration to be constrained to those of the form $(X_n)_{n \in C}$, where $C$ is a \emph{countably infinite} indexing set. This condition posses major restrictions to the immediate applicability of his theorem. In particular, Gaussian, gamma, and beta process fall outside the sphere of operability of his result.

 In this paper, which is a continuation of the work completed in \cite{Finn}, we address the issue of obtaining: (i)  a \emph{general} construction for conjugate priors over an infinite dimensional space, as well as (ii) a derivation of \emph{explicit} updating formulae for posterior parameters as functions of prior parameters and observed data. The first of these two goals will be achieved by generalizing \cite{Hjort}'s construction of the beta process via appropriate utilization of sufficient statistics for exponential families. To accomplish the second goal,  we show how to obtain a \emph{canonical} construction for triples of stochastic processes whose L\'{e}vy measure densities are from positive valued exponential families and which reflect the conjugacy relationship between prior, likelihood, and posterior. This canonical methodology will be a direct consequence of the construction of exponential family L\'{e}vy measure densities for nonparametric conjugate models achieved in the first part of our program. We will then demonstrate that the triples produced by this canonical construction in fact form the prior, likelihood, and posterior in conjugate positive valued exponential families when cast in the setting of conditional expectation operators. Updating formulae for the parameters governing the density of the processes' L\'{e}vy measure, analogous to the updating of the parameters of density of the random variable in the finite dimensional case will be achieved.  
\section{Statistical Preliminaries}

In this section we first recall the definition of a completely random measure, and provide a brief account of their connection to stochastic processes, as well as state the usual representation theorem for completely random measures in terms of Poisson processes. We then gather requisite facts regarding exponential families and their sufficient statistics, as well as provide an extension of a well known representation of the moments of the sufficient statistics to the setting of variational calculus. Finally, we turn to the measure theoretic necessity of providing a rigorous definition of conditional expectations and probabilities required for our framework.

\subsection{Completely random measures}

A random measure, $\Phi$, is a function whose domain is a measure space $(\Omega, \mathcal{F},\mu)$ and whose range is a space of measures over a state space $(S,\Sigma)$,  where we take $\Sigma$ to be a $\sigma$-algebra of subsets of $S$. In other words, for each $\omega$ $\in$ $\Omega$ we have $\Phi(\omega,\cdot)$ is a measure on $(S,\Sigma)$, and for each fixed $A$ $\in$ $\Sigma$, the function $\Phi(\cdot,A):\Omega \longrightarrow \mathbb{R}^+$ is an $\mathcal{F}$-measurable function. In this way we may view a random measure $\Phi$ as a collection of random variables over $\Omega$ indexed by the elements of $\Sigma$. 

 A random measure $\Phi$ is said to be a \emph{completely} random measure if for any finite collection $A_1,A_2,\ldots,A_n$ of elements of $\Sigma$ which are pairwise disjoint, the corresponding random variables $\Phi(\cdot,A_1), \Phi(\cdot,A_2),\ldots,\Phi(\cdot,A_n)$ are independent. To motivate a bit of intuition in this situation, take for example $S=[0,+\infty)$ and $\Sigma$ to be the collection of Borel measurable subsets of $S$. Then the condition of complete randomness implies for any $t_1 < t_2 < \ldots < t_n$ we have $\Phi(\cdot,(t_1,t_2]), \Phi(\cdot,(t_2,t_3]),\ldots,\Phi(\cdot,(t_{n-1},t_n])$ are independent. This is reminiscent of the situation where $X(t)$ is a nondecreasing stochastic process with independent increments in the sense that for any $t_1 < t_2 <\ldots < t_n$ we have $X(t_2)-X(t_1), X(t_3) - X(t_2),\ldots,X(t_n) - X(t_{n-1})$ are independent, and $\Phi$ is the unique Borel measure for which $\Phi((a,b]) = X(b+) - X(a+)$ for all $a < b$. As particular examples of this definition, both the beta and the gamma processes satisfy the conditions of a completely random measure.

\begin{figure}
\includegraphics[width=14cm, height=9cm]{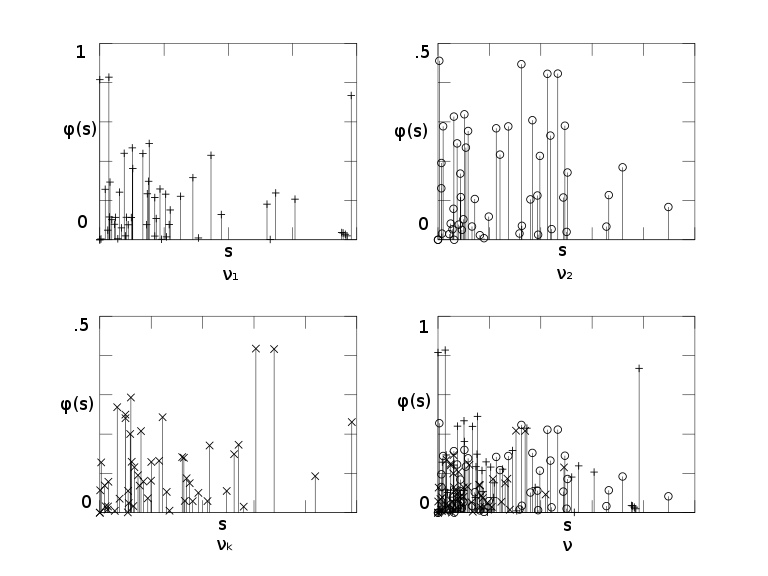}
\caption{\hspace{-.2cm} A L\'{e}vy measure $\nu$ decomposed into $\{\nu_n\}_{n=1}^{\infty}$. Each measure $\nu_n$ is constructed from a Poisson process, $\Pi_n$, on $S\times(0,\infty]$ yielding an atomic measure with with weights $\phi(s)$ assigned to the point mass at $s$ where $(s,\phi(s))\hspace{-.1cm} \in \Pi_n$. The L\'{e}vy measure $\nu$ is the superposition of the measures $\nu_1,\nu_2,\ldots,\nu_k,\ldots$.} 
\end{figure}

\cite{Kingman67,Kingman93} showed that any completely random measure $\Phi$ has a decomposition into independent completely random measures of the form $\Phi = \Phi_f + \Phi_d + \Phi_o$, where $\Phi_f$ corresponds to the fixed atoms of $\Phi$, $\Phi_d$ is the deterministic component of $\Phi$, and $\Phi_o$ is a purely atomic measure. In general, it is the measure $\Phi_o$ that is of interest. One can show that for any $A\in\Sigma$ the random variable $\Phi_o(\cdot,A)$ is infinitely divisible in the sense that for any $n$ there exists a decomposition of $A$ into pairwise disjoint sets $A_1, A_2,\ldots,A_n\in\Sigma$ such that
\begin{equation*}
\mathbb{E}[\exp(-\Phi_o(A_i))] = \{\mathbb{E}[\exp(-\Phi_o(A))]\}^{\frac{1}{n}}\hspace{.2cm}i=1,\ldots,n.
\end{equation*}
This property implies that the transform $\mathbb{E}[e^{-t\Phi_o(A)}]$ has the form
\begin{equation*}
\mathbb{E}[e^{-t\Phi_o(A)}]=\exp\bigg(-\int_{A\times (0,\infty]}(1-e^{-ts})\nu(dx,ds)\bigg).
\end{equation*}

The measure $\nu$ constructed in \cite{Levy} is now commonly referred to as the L\'{e}vy measure  of the random variable $\Phi_o(A)$ and is of great importance as it determines the random variable $\Phi_o(\cdot,A)$.

Kingman proved the L\'{e}vy measure $\nu$ and the completely random measure $\Phi_o$ have parallel decompositions of the form 
\begin{equation*}
\nu = \sum_n \nu_n, \hspace{1cm}\Phi_0 = \sum_n \Phi_n, \mbox{\hspace{.3cm}where for each }n,\hspace{.3cm} \Phi_n = \sum_{(s,\phi(s)) \in \Pi_n}\phi(s)\delta_s.
\end{equation*}
Here, for each index $n$, $\Pi_n$ is a Poisson process on $S\times(0,\infty]$, and $s \in S$ is an atom of weight $\phi(s) \in (0,\infty]$. It is precisely this decomposition derived from the L\'{e}vy measure that permits simulation of the completely random measure via simulation of the Poisson processes which constitute the decomposition, where each $\nu_n$ is used as the mean measure  of the Poisson process $\Pi_n$. The decomposition which allows one to view the completely random measure as a composition of atomic measures defined by Poisson processes is illustrated in Figure 1. \cite{Wang} gave a detailed analysis of this decomposition and subsequent simulation techniques for the L\'{e}vy measures arising from the beta and gamma processes.

\subsection{Exponential families}

We now turn to a brief accounting of exponential families and properties of their sufficient statistics. Standard references for material on exponential families are \cite{TPE} and \cite{TSH}.

A family $\{P_{\theta}\}_{\theta \in \Theta}$ of distributions over a probability space $(\Omega, \mathcal{F})$ is said to constitute an $n$-dimensional exponential family if the distributions have densities of the form $p_{\theta}(x)=h(x)e^{(\langle\eta(\theta),T(x)\rangle-B(\theta))}\mu(dx)$ with respect to some common measure $\mu$. In the above, $\eta(\theta)=(\eta_1(\theta),\ldots,\eta_n(\theta))$, and $\eta_i$ and $B$ are real-valued functions of the parameter $\theta$. In addition,  $T(x)=(T_1(x),\ldots,T_n(x))$, the $T_i$ are real-valued statistics where $x$ is in  the support of the density, and $\langle\eta,T(x)\rangle$ denotes the usual inner product of $\eta$ and $T(x)$. While this is the formal definition of an exponential family, the form commonly used is obtained by employing the $\eta_i$, $i=1\ldots n$,  as the parameters and writing the density in what is known as the \emph{canonical form} $p(x|\eta)=h(x)e^{(\langle\eta,T(x)\rangle-A(\eta))}\mu(dx)$. The integrand of the density in its canonical form is a positive function and will yield a bona fide probability distribution if and only if 
\begin{equation*}
\int h(x)e^{(\langle\eta,T(x)\rangle-A(\eta))}\mu(dx) = 1, \hspace{.1cm}\mbox{which is equivalent to} 
\end{equation*}
\begin{equation*}
\int h(x)e^{\langle\eta,T(x)\rangle}\mu(dx)=e^{A(\eta)} < +\infty.
\end{equation*}
The collection of all such $\eta$ for which this holds is a convex set called the \emph{natural parameter space} and is denoted by $\Xi$. Many common distributions belong to exponential families; for example, the normal, beta, and gamma distributions are all members of the exponential family of distributions.

The real valued functions, i.e. the statistics,  $T_i(x)$, $i=1\ldots n$ appearing in the expression for the densities of an exponential family posses a pleasant property known as \emph{sufficiency}. A statistic $T$ for a random observation $X$ is said to be \emph{sufficient} for the family $\{P_{\theta}\}$ of possible distributions for $X$ if the conditional expectation of $X$ given $T=t$ is independent of $\theta$ for all $t$. The property of sufficiency has numerous and varied consequences which, although of great importance in many branches of statistics, only one of which will concern us in this paper. This property allows the computation of moments of $T$, a variant of which will supply a crucial step in the proof of our extension of Hjort's result. This well-known result relates the moments of $T$ to the partial derivatives of the normalizing factor $\exp(A(\eta))$ with respect to the natural parameters. In fact, it is shown in \cite{TPE} that for any $k$ where $1\leq k \leq n$ and any $m \geq 1$ we have 
\begin{equation}\label{eq0}
\mathbb{E}[T_k^m]=\int T_k^m(x)(h(x)\exp\big(\langle\eta,T(x)\rangle-A(\eta)\big)\mu(dx) = e^{-A(\eta)}\frac{\partial^m[e^{A(\eta)}]}{\partial \eta_k^m}.
\end{equation}
We require a variant of this result which replaces the natural parameter $\eta$ by a function $\eta(z)$ defined on $(0,\infty)$. Thus, rather than taking the natural parameters to freely vary we require them to be in the range of the function $\eta(z)$. Given this modification, the variant of \eqref{eq0} involves functional derivatives, a common device from variational calculus. The result in Lemma 1 follows from the definition of and standard properties of the functional derivative. For the sake of completeness we briefly provide the definition of a functional derivative below. A standard reference for the definitions that follow is \cite{GelfandFomin}.

Given a space of functions, say $\mathcal{B}$ is the space of piecewise continuous functions on $(0,\infty)$, and a functional $F:\mathcal{B} \longrightarrow \mathbb{R}$, the functional derivative of $F$ with respect to $a(x) \in \mathcal{B}$ is defined as
\begin{equation*}
\frac{\partial F[a(x)]}{\partial a(x)} = \lim_{\varepsilon \to 0} \frac{F[a(x)+ \varepsilon \delta(x)] - F[a(x)]}{\varepsilon},
\end{equation*}
where $\delta(x)$ is an element of a class of test functions, usually taken to be a class of indicator functions or the class of bump functions on the domain of $\mathcal{B}$.

The statement of the result is as follows:
\newtheorem{Lemma}{Lemma}
\begin{Lemma}
Let $\eta(z)=(\eta_1(z),\ldots,\eta_n(z))$ be a piecewise continuous function on $(0,\infty)$ such that $\eta((0,\infty)) \subseteq$\hspace{.1cm}$\Xi$\hspace{.05cm}, where \hspace{.1cm}$\Xi$\hspace{.1cm} is the nonempty natural parameter space of the exponential family given by

\begin{equation*}
p(x|\eta)_{\theta}=h(x)\exp\big(\langle \eta, T(x)\rangle-A(\eta)\big)\mu(dx).
\end{equation*}
\noindent
Then the functional $\mbox{m}^{th}$-derivative with respect to $\eta_k(z)$, $1 \leq k \leq n$ exists and the $\mbox{m}^{th}$ moment of $T_k(x)$ has the form
\begin{equation}\label{eq1}
\mathbb{E}[T_k^m] = e^{-A(\eta(z))}\frac{\partial^m}{\partial(\eta_k(z))^m}\big[e^{A(\eta(z))}\big].
\end{equation}
\end{Lemma}

\subsection{Conditional Probabilities}
The classical notion of conjugacy is most commonly expressed in the language of conditional probabilities, which, in the finite dimensional case can be stated quite simply. In fact, given a probability space $(\Omega,\Sigma,P)$ and events $A,B \in \Sigma$, the conditional probability of $A$ given $B$, denoted $P(A|B)$, is defined as $\frac{P(A \cap B)}{P(B)}$, subject to the obvious restriction that $P(B)\neq 0$. In the infinite dimensional case this definition is still operable, albeit quite restrictive, as many interesting events may in fact have probability zero. This phenomenon is witnessed by the case that $(X_t)_{t \geq 0}$ is a continuous valued stochastic process, e.g. Brownian motion, and the event $B$ is an observation of the value of the process at some time $t_0 > 0$. 

The obstacle posed by the desire to divide by an event of probability zero can be outflanked by an application of the \emph{Radon-Nikod\'{y}m} theorem which serves as the crucial mathematical tool in successfully defining a conditional expectation operator on a space of integrable functions over a general measure space $(\Omega,\mathcal{F},\mu)$. Although one can easily achieve this level of generality, we will restrict our discussion in the sequel to the case of probability spaces $(\Omega,\Sigma,P)$.

We begin our formal definition of conditional probabilities by first defining the notional of a conditional expectation operator, and then via this definition pass to the definition of a conditional probability as a particular case of the latter. For the reader who is either unfamiliar with this topic, or desires a more in depth treatment of the topic than we are able to provide within, we note that the presentation in this section is based on the excellent source \cite{Rao}. 

Continuing on, let $(\Omega,\Sigma,P)$ be a probability space, $f:\Omega \longrightarrow \mathbb{R}$ an integrable random variable, and $\mathcal{B} \subset \Sigma$ a $\sigma$-algebra. Considering 
\begin{center}
$\nu_f: A \longrightarrow \displaystyle\int_A fdP$,\hspace{.2cm}$A \in \mathcal{B}$
\end{center}

\noindent
we have a $\sigma$-additive set function on $\mathcal{B}$ such that $\nu_f \ll P_{\mathcal{B}}$, i.e. $\nu_f$ is absolutely continuous with respect to  $P_{\mathcal{B}}$, which denotes $P$ restricted to $\mathcal{B}$. Thus, by the Radon-Nikod\'{y}m theorem there exists a $P_{\mathcal{B}}$-unique integrable random variable $\tilde{f}$ on $(\Omega,\mathcal{B},P_{\mathcal{B}})$ such that  
\begin{center}
$\nu_f(A)=\displaystyle\int_A \tilde{f}dP_{\mathcal{B}}$,\hspace{.2cm}$A \in \mathcal{B}$.
\end{center}

The random variable $\tilde{f}$ is \emph{defined} to be the conditional expectation of $f$ given $\mathcal{B}$ and is denoted by $E(f|\mathcal{B})$. We derive the definition of the conditional probability of $A \in \Sigma$  given $\mathcal{B}$ by simply applying the above definition to the characteristic function of $A$, i.e. $P(A|\mathcal{B})=E(\chi_A|\mathcal{B})$. Note that with this definition of a conditional probability we have the intuitive formula

\begin{center}
$\displaystyle\int_B P(A|\mathcal{B})dP_{\mathcal{B}}=\displaystyle\int_B \chi_AdP=P(A\cap B),\hspace{.2cm}A \in \Sigma,\hspace{.2cm}B \in \mathcal{B}$.
\end{center}

While the above definition provides a mathematically sound framework for the notion of conditional probability, what we require at this point is a methodology to apply this machinery to the notion of  $P(A| X=x_0)$ for a random variable $X$, even though the event $\{X=x_0\}$ is not a $\sigma$-algebra. Closely related to the event $\{X=x_0\}$ is the $\sigma$-algebra $X^{-1}(\mathcal{B}(\mathbb{R}))$ where $\mathcal{B}(\mathbb{R})$ is the Borel $\sigma$-algebra on $\mathbb{R}$. We thus turn to an alternative definition of conditional expectations, which, when combined with the \emph{Doob-Dynkin lemma} will provide the crucial bridge between $P(A| X=x_0)$ and $P(A|X^{-1}(\mathcal{B}(\mathbb{R}))$.

To that end, let $(\Omega,\Sigma,P)$ be a probability space, $(S,\mathcal{S})$ a measurable space, $h:\Omega \longrightarrow \mathbb{R}$ a $(\Sigma,\mathcal{S})$ measurable function, and $\mathcal{B}=h^{-1}(\mathcal{S})$. For any $P$-integrable function, $f$, let $\tilde{\nu_f}=\nu_f \circ h^{-1}$. Then $\tilde{\nu_f}$ is $\sigma$-additive on $(S,\mathcal{S})$ and defining $\mu = P \circ h^{-1}: \mathcal{S} \longrightarrow \mathbb{R}$ we have $\tilde{\nu_f} \ll \mu$. Hence, by the Radon-Nikod\'{y}m theorem there exists a $\mu$-unique integrable function $\tilde{g_f}: \mathcal{S} \longrightarrow \mathbb{R}$ such that

\begin{center}
$\tilde{\nu_f}(A)= \displaystyle\int_A \tilde{g_f}d\mu=\nu_f(h^{-1}(A))=\displaystyle\int_{h^{-1}(A)}fdP ,\hspace{.2cm}A \in \mathcal{S}$.
\end{center}

 The random variable $\tilde{g_f}$ is \emph{defined} to be the conditional expectation of $f$ given $h$ and is denoted by $E(f|h=s),\hspace{.2cm}s \in S$. Similar to the previous case, we derive the definition of the conditional probability of $A \in \Sigma$  given $h$ by applying the above definition to the characteristic function of $A$, i.e. $P(A|h=s)=E(\chi_A|h=s),\hspace{.2cm}s \in S$. We may easily obtain the first definition of conditional expectations from the second by simply letting $S=\Omega$, $\mathcal{S}=\mathcal{B} \subset \Sigma$, and $h(s)=s$.

In order to complete the bridge between $P(A| X=x_0)$ and $P(A|X^{-1}(\mathcal{B}(\mathbb{R}))$ we first remind the reader of the \emph{Doob-Dynkin lemma}.

\begin{Lemma}
Let $(\Omega_i,\Sigma_i)$, $i=1,2$ be measurable spaces and $f:\Omega_1 \longrightarrow \Omega_2$ be $(\Sigma_1,\Sigma_2)$ measurable. Then any function $g:\Omega_1 \longrightarrow \mathbb{R}$ is $f^{-1}(\Sigma_2)$-measurable iff $g=h \circ f$ for some measurable $h:\Omega_2 \longrightarrow \mathbb{R}$.
\end{Lemma}

Employing the Doob-Dynkin lemma in our cause, let $h:\Omega_1 \longrightarrow \mathbb{R}$ a $(\Sigma_1,\mathcal{B}(\mathbb{R}))$ measurable function, $\Omega_2=\mathbb{R}$, and $\Sigma_2=\mathcal{B}(\mathbb{R})$. We then have $E(f|h)=g \circ h$ for some $g:\mathbb{R} \longrightarrow \mathbb{R}$ a Borel function, where $E(f|h)$ denotes $E(f|h^{-1}(\mathcal{B}(\mathbb{R})))$. Hence $E(f|h)$ is a Borel function of $h$, and thus $E(f|h^{-1}(\mathcal{B}(\mathbb{R}))(s)=E(f|h=s)=\varphi(h(s)),\hspace{.2cm}s \in S$ for some Borel function $\varphi:\mathbb{R} \longrightarrow \mathbb{R}$. The upshot of this is that $E(f|h=s)$ is \emph{well defined} as a measurable function \emph{regardless} of the measure of the event $\{\omega \in \Omega|\hspace{.1cm}h(\omega)=s\}$. Note that $E(f|h=s_0)$ is the value of the measurable function $E(f|h^{-1}(\mathcal{B}(\mathbb{R}))(s)$ at $s_0$.

\section{Sufficient statistics construction of a L\'{e}vy measure}

We begin part three of this paper with a brief discussion of our motivation for presenting the central construction, and in particular the role we intend this construction to take as the first of two components necessary to extend the researcher's palette of nonparametric conjugate models. Following this, in the first of the two remaining sections we discuss Hjort's construction of the beta process. The conditions under which the construction is implemented and an essential step in his proof are explained. Our modifications of these which allow the construction to produce processes whose infinitesimal increments are distributed according to an exponential family are provided. In the final section we state and interpret our main result.

\subsection{Conjugacy and densities}
To begin, we remind the reader of the brief discussion contained in the introduction regarding the forms conjugacy assumes in the parametric and nonparametric cases. In the former, conjugacy is well known and accepted as the condition that the density of the posterior is \emph{of the same form} as the density of the prior, \emph{and} the parameters for the posterior are obtained as a function of the parameters for the prior and the sampled data. In contrast, for the nonparametric case, the form conjugacy now commonly assumes applies not to the densities of the random variables under consideration, but to the densities of the L\'{e}vy measures associated with the processes. 

For example, the beta process is taken to be a stochastic process whose infinitesimal increments are, for a given base measure $\mu$ and concentration function $c(\omega)$, BP$(c(d\omega)\mu(\omega),c(\omega)(1-\mu(\omega))$ distributed. This somewhat imprecise definition of the beta process is formalized via the fact that the associated L\'{e}vy measure of the process has the form $\nu(d\pi,d\omega) = c(\omega)\pi^{-1}(1-\pi^{c(\omega)-1})d\pi \mu(d\omega)$, a form which is a degenerate beta density. To obtain the data generating process for a conjugate beta process model, one invokes the well known conjugacy of the beta and Bernoulli distributions from the parametric case, and then pushes the analogy to the infinite dimensional case. In fact, letting $X$ be a Bernoulli process with base measure $B$, denoted as $X \sim \mbox{Be}(B)$, one can show that if $B \sim \mbox{BP}(c,B_0)$ and $X_1|B, X_2|B,\ldots,X_n|B \sim \mbox{Be}(B)$ are independent observations, then
\begin{equation*}
 B|X_1,\ldots,X_n \sim \mbox{BP}\left(c+n, \frac{c}{c+n}B_0 + \frac{1}{c+n}\sum X_i\right)
\end{equation*} 
This result, while pointed out by \cite{Thibaux}, derives its proof from a result of \cite{Hjort} which in fact proves that the density of the L\'{e}vy measure associated with the posterior maintains the form of a beta density with parameters derived from the prior and the observed data.

Thus, in both the parametric and nonparametric cases, the active definition of conjugacy is one of a condition on the relationship between densities associated with the prior and posterior. In the parametric case the definition involves the densities of the random variables, and in the nonparametric case the definition involves the densities of the L\'{e}vy measures associated with the stochastic processes.

It is precisely this observation which drives us to consider extending Hjort's construction to encompass stochastic processes whose infinitesimal increments are distributed according to a positive exponential family. Or more precisely, to stochastic processes whose associated L\'{e}vy measures have densities from positive exponential families. We note here that we use the term \emph{positive exponential family} to denote an exponential family distribution which takes on positive values in its state space and has a least one sufficient statistic which does not change sign over the state space. Note that this property is indeed realized by the majority of common exponential families taking positive values in their respective state spaces. In fact, for the Bernoulli, beta, binomial, negative binomial, exponential, gamma, inverse gamma, inverse chi-squared, log-normal, Pareto (with known $\alpha_m > 1$), and Poisson distributions, each has a sufficient statistic which realizes this property. For the duration of this paper all exponential families under consideration are assumed to be positive exponential families.

 This choice of density class takes aim directly at the problem which concerns the second part of our program, namely demonstrating that the processes yielded by our construction in this paper, do in fact, under appropriate regularity conditions, serve as \emph{infinite dimensional conjugate priors}.

With a view to the second part of our program in this paper, i.e. the establishment of explicit updating formulae for posterior parameters as functions of prior parameters and observed data, we point out the following observation. The above formula for updating the posterior parameters of the beta process given the data $X_1,\ldots,X_n$ and the prior parameters $c$ and $B_0$ in fact \emph{precisely} mirrors the formula for updating the posterior parameters given data and prior parameters in the finite dimensional case. This is not a coincidence. Indeed, in all known infinite dimensional cases involving exponential family L\'{e}vy measure densities, i.e.  gamma-Poisson, beta-Bernoulli, and beta-negative binomial processes, this is certainly the case. The updating formulae in the infinite dimensional case mirrors that of the finite dimensional case. The organizing principle behind this pleasant fact is illuminated in the conditions and proof set forth for our main theorem in section 4.

\subsection{Generalizing Hjort's construction of the improper beta process}

\cite{Hjort} provides an explicit construction proving the existence of a process $B(t)$ on $[0,\infty)$ such that $B(0)=0$ and $B(t)$ possesses independent increments which are infinitesimally beta distributed. In addition, Hjort requires the condition that the sample paths of $(1-e^{-B(t)})$ are all cumulative distribution functions. The construction is given relative to two fixed objects: 1. a nondecreasing, right continuous function, $A_0(t)$ on $[0,\infty)$ with $A_0(0)=0$ and the quantity $(1-e^{-A_0(t)})$ yielding a cumulative distribution function on $[0,\infty)$, and 2. a piecewise continuous function $c(z)$ on $(0,\infty)$. The properties required of the function $A_0(t)$ ensured the resulting process $(1-e^{-B(t)})$ would have sample paths that were cumulative distribution functions. The function $c(z)$, which is termed the \emph{concentration function}, in part determines the beta distribution of the increments of $B(t)$.

Proving the correctness of the form for the L\'{e}vy measure relies heavily on the ability to find a convenient closed form for all moments of a beta distributed random variable. Fortunately it is known that for $X$ $\sim \mbox{Beta}(\alpha,\beta)$ we have for all $m\geq1$
\begin{equation*}\label{eq2}
\mathbb{E}[X^m]=\frac{\Gamma(\alpha + m)\Gamma(\alpha + \beta)}{\Gamma(\alpha + \beta + m)\Gamma(\alpha)}.
\end{equation*} 

When attempting to generalize Hjort's construction to prove the existence of a process $X(t)$ on $[0,\infty)$ such that $X(0)=0$ and  $X$ possesses independent infinitesimal increments which are distributed according to an exponential family, the lack of a closed form for the moments of a random variable distributed according to a general exponential family imposes a significant obstacle. However, equation \eqref{eq1} provides a formula for the moments of the sufficient statistics $T_k(x)$ relative to a density which is a modification of the density of $X$. It is for this reason that our extension of Hjort's construction uses a sufficient statistic $T_k(x)$ of $X$ rather than $X$ itself.

Additionally, we do not require the fixed function $A_0(t)$ to have the property that $(1-e^{-A_0(t)})$ yields a cumulative distribution function on $[0,\infty)$. We only require that $A_0(t)$ corresponds to a unique Lebesgue-Stieltjes measure on $[0,\infty)$. Finally, we replace the piecewise continuous function $c(z)$ with the vector of piecewise continuous functions $(\eta_1(z),\dots,\eta_n(z))$.

Of the number of noteworthy characteristics possessed by the stochastic process yielded by Hjort's construction, we will now discuss a few in order to clarify our construction and call attention to some subtle points which may be obscured by the current literature on beta processes.

First, it should be made clear that Hjort's definition of a L\'{e}vy process is that of a stochastic process which is a ``nonnegative, nondecreasing processes on $[0,\infty)$ that start[s] at zero and ha[s] independent increments \cite{Hjort}.'' This definition stands in contrast to the definition of L\'{e}vy processes given in standard texts such as \cite{Sato} and \cite{Kallenberg}. For example, the definition of a L\'{e}vy process from \cite{Sato} requires the process to start at zero almost surely as well as have increments which are independent \emph{and} stationary. Processes such as those \emph{termed} L\'{e}vy processes in \cite{Hjort} are referred to as \emph{additive processes} in \cite{Sato}. One can see that, unless $A_0(t)$ is a linear function, Hjort's beta process, $A(t)$, cannot have stationary increments by  Hjort's observation that $\mathbb{E}[A(a,b]]=A_0(a,b]$.

Second, Hjort's construction of the beta process is such that the resulting process possesses infinitely many jumps in a given finite interval. That this is the case is best illustrated by briefly recalling the construction employed. For the given concentration function $c(z)$, and the initial measure, $A_0(z)$, Hjort defines random variables, $X_{n,i}$, for each pair $n,i$, which are properly beta$(a_{n,i},b_{n,i})$ distributed, where $a_{n,i}=c_{n,i}A_0(\frac{i-1}{n},\frac{i}{n}]$, $b_{n,i}=c_{n,i}(1-A_0(\frac{i-1}{n},\frac{i}{n}])$, and $c_{n,i}=c(\frac{i-\frac{1}{2}}{n})$. It is the sums
\begin{center}
$A_n(t)=\displaystyle\sum_{\frac{i}{n}\leq t}X_{n,i}$
\end{center}

\noindent
taken as $n \longrightarrow \infty$ that converge in distribution to the beta process which is the target of the construction. This provides, as noted by Hjort, $A_n$ which have ``independent beta increments and its jumps become smaller, but occur more often, as $n$ increases \cite{Hjort}.'' This fact arises from the consideration of the beta distribution. For a beta$(\alpha,\beta)$ distributed random variable, $X$, it is a simple computation to show that $\mathbb{E}[X]=\frac{\alpha}{\alpha + \beta}$. Given the choice of $a_{n,i}$ and $b_{n,i}$, if $X$ is beta$(a_{n,i},b_{n,i})$ distributed, then  $\mathbb{E}[X]=\frac{a_{n,i}}{a_{n,i} + b_{n,i}}=A_0(\frac{i-1}{n},\frac{i}{n}] \longrightarrow 0$ as $n \longrightarrow \infty$. It is thus Hjort's choice of defining the parameters of the exponential family to be functions of the initial measures of the time increments, i.e. $A_0(\frac{i-1}{n},\frac{i}{n}]$, that achieves the phenomenon of infinitely many jumps in a finite interval.  

As we are providing a general construction for exponential families, our proof employs the same idea, albeit using a slightly different procedure. We first let $\{S_{n,i}\}_{n,i\in \mathbb{N}^2}$ be a collection independent random variables, of which each $S_{n,i}$ is distributed according to the exponential family $p(s|\eta_{n,i})$, where  $\eta_{n,i}=\eta\big(\frac{i-\frac{1}{2}}{n}\big)$. Then employing the $k$th sufficient statistic of the exponential family, $T_k$, define $A_{0,n,i}=A_0(\frac{i-1}{n},\frac{i}{n}]$ and $T_{k,n,i}=A_{0,n,i}T_k(S_{n,i})$. We thus obtain $\{T_{k,n,i}\}_{n,i\in \mathbb{N}^2}$, a collection of independent random variables distributed according to $p(T^{-1}_{k,n,i}|\eta_{n,i})\frac{dT_k^{-1}}{du}$ where $T_{k,n,i}(s)=u$ and it is the sums
\begin{center}
$T_{k,n}(t)=\displaystyle\sum_{\frac{i}{n}\leq t}T_{k,n,i}$
\end{center}

\noindent
taken as $n \longrightarrow \infty$ that converge in distribution to the stochastic process which is the target of our construction. Once again, the $T_{k,n}$  have independent increments from an exponential family, and their jumps become smaller, but occur more often, as $n$ increases. This fact arises not from consideration of the expected value of the proper exponential family, but rather from the scaling of the random variables $T_k(s)$ by the factor $A_0(\frac{i-1}{n},\frac{i}{n}]$. Although this may seem somewhat distant from Hjort's approach, this is not the case. In fact, considering the form of an exponential family for a beta distribution appearing in Hjort's construction 
\begin{equation*}
\mbox{exp}\Big\{c_{n,i}A_0\Big(\frac{i-1}{n},\frac{i}{n}\Big]\mbox{ln}(s)+ c_{n,i}\Big(1-A_0\Big(\frac{i-1}{n},\frac{i}{n}\Big]\Big)\mbox{ln}(s-1)-A(a_{n,i},b_{n,i})\Big\}=
\end{equation*}
\begin{equation*}
\mbox{exp}\Big\{c_{n,i}A_0\Big(\frac{i-1}{n},\frac{i}{n}\Big]\mbox{ln}(s) + c_{n,i}\mbox{ln}(s-1) + c_{n,i}A_0\Big(\frac{i-1}{n},\frac{i}{n}\Big]\mbox{ln}(s-1)-A(a_{n,i},b_{n,i})\Big\}
\end{equation*}
\noindent
we see that the terms $c_{n,i}A_0(\frac{i-1}{n},\frac{i}{n}]\mbox{ln}(s)$ and $A_0(\frac{i-1}{n},\frac{i}{n}]\mbox{ln}(s-1)$ can be viewed as the product of the parameter $c_{n,i}$ and the quantities $A_0(\frac{i-1}{n},\frac{i}{n}]\mbox{ln}(s)$ and $A_0(\frac{i-1}{n},\frac{i}{n}]\mbox{ln}(s-1)$ respectively. These quantities are analogous to the summands $T_{k,n,i}=A_{0,n,i}T_k(S_{n,i})$ appearing in our construction.

The fact that our construction provides a generalization to all proper exponential families of Hjort's construction for beta distributed random variables which maintains the desirable property of infinitely many jumps in finite intervals emphasizes the non-triviality of our resulting processes. Indeed, the L\'{e}vy measures under consideration in Theorem 3 in this section are \emph{proper} in the sense that the densities associated with the L\'{e}vy measures are from proper exponential families. This stands in direct contrast to the \emph{improper} beta density which is associated with the beta process, and may give some readers reason to pause. This property of the L\'{e}vy measures under present consideration does \emph{not} imply that the resulting stochastic processes are necessarily merely compound Poisson processes. To illustrate that this need not be the case, we remind the reader of the following definition. 
\newtheorem{Definition}{Definition}
\begin{Definition}
A distribution $\mu$ on $\mathbb{R}$ is said to be a compound Poisson process if, $\exists \hspace{.1cm}c>0$ and a distribution $\sigma$ on $\mathbb{R}$ such that $\sigma({0})=0$ and 
\begin{equation*}
\widehat{\mu} (dx)=e^{c(\widehat{\sigma}(z)-1)}, z \in \mathbb{R}.
\end{equation*}
where, for a given distribution $\gamma$, $\widehat{\gamma}$ denotes the characteristic function of $\gamma$.
\end{Definition}

\noindent With the above definition in hand we may now state the following characterization of compound Poisson processes with associated distributions $\sigma$ on $(0,\infty)$. This result may be found in \cite{Sato}.

\begin{Lemma}
Let $c > 0$ and let $\sigma$ be a distribution on $(0,\infty)$. Let $\{X_t\}$ be the compound Poisson process on $\mathbb{R}$ associated with $c$ and $\sigma$. Then $X_t$ is increasing in $t$ almost surely, and
\end{Lemma}
\begin{center}
$\mathbb{E}[e^{-uX_t}]=\mbox{exp}\Big \{tc \displaystyle\int_{(0,\infty)}(e^{-ux}-1)\sigma(dx)\Big \}\hspace{.3cm}\mbox{for}\hspace{.2cm} u \geq 0.$
\end{center}

\noindent
Lemma 3 is sufficient to prove that our construction is capable of yielding stochastic processes which are not compound Poisson processes. For example, if one considers the Pareto exponential family with density given by  $p(u|\alpha) = \frac{\alpha u_m^{\alpha}}{u^{\alpha + 1}}$, taking $\alpha(z)=z$ and $A_0(z)$ to be standard Lebesgue measure on $(0,\infty)$, then the stochastic process, $T(t)$, which is a consequence of Theorem 4 is in fact \emph{not} a compound Poisson process. We will detail this computation following the statement and proof of Theorem 4 after we have introduced the necessary form for the L\'{e}vy measure corresponding to the process $T(t)$.

\subsection{Sufficient statistics construction for positive valued exponential families}

We now state and prove our main result. In the light of the discussion of Hjort's construction above, the following paragraph provides a few essential comments regarding the construction obtained in Theorem 3 and the objects it produces.

For a given $l$-dimensional exponential family, $p(x|\eta)$, defined on $[0,\infty)$, one has, in our case, the function $\eta(z)=(\eta_1(z),\ldots,\eta_l(z))$ supplying the natural parameters, and the sufficient statistics $T(s) = (T_1(s),\ldots,T_l(s))$. Our construction employs only \emph{one} of the sufficient statistics, say $T_k(s)$ for a fixed $k$ where $1 \leq k \leq l$. As long as the chosen $T_k(s)$ and the given function $\eta(z)$ satisfy the regularity conditions in the theorem, then for any positive, increasing $A_0(z)$ on $[0,\infty)$ which is right continuous with left hand limits, the construction of the theorem results in a L\'{e}vy process $T(t)$ and a corresponding L\'{e}vy measure given by the specified form. Indeed, the change of variables formula expressed in equation (3) in the theorem furnishes the machinery to accomplish this task. That this change of variables  permits such a construction is shown in the example 4.3. 

\newtheorem{Theorem}{Theorem}
\begin{Theorem}\label{MainTheorem1}
Let $T_k(x)$ be a sufficient statistic of an $l$-dimensional exponential family $p(x|\eta)$, $\eta(z)=(\eta_1(z),\ldots,\eta_l(z))$ a vector of piecewise continuous, nonnegative functions on $(0,\infty)$, and $A_0(z)$ a positive, increasing function on $[0,\infty)$ which is right continuous with left hand limits. Assume the following conditions hold:

\begin{enumerate} 
\item $T_k^{-1}(u)$ exists and is differentiable;
\item for all $z\hspace{.1cm} \in \hspace{.1cm} (0,\infty) \hspace{.1cm}$ we have $\eta(z) \in \Xi\hspace{.1cm}$, the natural parameter space of $p(x|\eta)$;
\item if $(\eta_1(z),\ldots,\eta_l(z)) \in \Xi$\hspace{.2cm}then for every $ 0<\varepsilon <1$ \hspace{.1cm}it follows that \newline$(\eta_1(z),\ldots,\varepsilon \eta_k(z),\ldots,\eta_l(z)) \in \Xi$.
\end{enumerate}
Then there exists a L\'{e}vy process $T(t)$ with a L\'{e}vy representation given by
\begin{equation*}
\mathbb{E}[\exp(-\theta T)] = \exp\left\{-\int(1-e^{-\theta u})dL_t(u)\right\} = \exp\left\{-\int(1-e^{-\theta T_k(s)})dL_t(s)\right\},
\end{equation*} 
\small
\begin{equation}\label{eq3}
\mbox{where}\hspace{.3cm}dL_t(u) = \left\{\int_0^t \exp\big(\langle\eta(z),U\rangle-A(\eta(z))\big)\frac{dT_k^{-1}}{du}dA_0(z)\right\}du
\end{equation}\\[.25cm]
\normalsize
\begin{equation}\label{eq4}
\mbox{and}\hspace{.3cm}dL_t(s) = \left\{\int_0^t \exp\big(\langle\eta(z),T(s)\rangle-A(\eta(z))\big)dA_0(z)\right\}ds,
\end{equation}\\[.25cm]
where $U=(T_1(T_k^{-1}(u)),...,T_l(T_k^{-1}(u)))$ and $u=T_k(s)$.\\[.25cm]
\end{Theorem}

\begin{proof}[Proof and Construction]
Let $\{S_{n,i}\}_{n,i\in \mathbb{N}^2}$ be a collection independent random variables, of which each $S_{n,i}$ is distributed according to the exponential family $p(s|\eta_{n,i})$, where  $\eta_{n,i}=\eta\big(\frac{i-\frac{1}{2}}{n}\big)$. Employing the $k$th sufficient statistic of the exponential family, $T_k$, and denoting  $A_0(\frac{i}{n})-A_0(\frac{i-1}{n}+)$ by $A_0(\frac{i-1}{n},\frac{i}{n}]$, define $A_{0,n,i}=A_0(\frac{i-1}{n},\frac{i}{n}]$ and $T_{k,n,i}=A_{0,n,i}T_k(S_{n,i})$. We thus obtain $\{T_{k,n,i}\}_{n,i\in \mathbb{N}^2}$, a collection of independent random variables distributed according to $p(T^{-1}_{k,n,i}|\eta_{n,i})\frac{dT_k^{-1}}{du}$ where $T_{k,n,i}(s)=u$.

\noindent
 Next define
\begin{equation*}
T_{k,n}(0)=0  \hspace{.1cm}\text{and} \hspace{.1cm} T_{k,n}(t)=\sum_{\frac{i}{n}\leq t}T_{k,n,i}(t) \hspace{.1cm} \text{for} \hspace{.1cm} t \geq 0
\end{equation*}
\noindent
so that with this definition, for every $n$, $T_{k,n}$ has independent increments from an exponential family. In fact, each $T_{k,n}$ is a nondecreasing random variable on $\mathbb{R}^+$, and thus may be considered as a random measure on $\mathbb{R}^+$. Hence, for each $(t_0,t_1] \subset \mathbb{R}^+$, we view $T_{k,n}(t_0,t_1]$ as random variable. The majority of this proof will be dedicated to showing that the Laplace transforms of the sums $\sum^k_{i=1}\theta_iT_{k,n}(t_{i-1},t_i]$, $\theta_i > 0 \hspace{.1cm} \forall i$, converge as $n \longrightarrow \infty$ to the quantity 

\begin{equation*}
\exp\bigg\{- \sum^k_{i=1 }\int(1-e^{\theta T_k(s)})dL_{(t_{i-1},t_i]}(s)\bigg\}
\end{equation*}
\noindent
from which we may quickly deduce the existence and stated properties of the stochastic process $T(t)$ by appealing to Theorem 16.16 in \cite{Kallenberg} and the standard Cram\'{e}r-Wold device. To that end we now proceed.

\noindent
 Performing a Taylor expansion on $e^{-\theta T_k(s)}$ yields
\begin{equation*}
-\int(1-e^{-\theta T_k(s)})dL_t(s) = -\int\bigg(1-\bigg(\sum_{m=0}^{\infty}\frac{(-1)^m\theta^mT_k^m(s)}{m!}\bigg)\bigg)dL_t(s)
\end{equation*}
\begin{equation*}
=\sum_{m=1}^{\infty}\frac{(-1)^m\theta^m}{m!}\int T_k^m(s)dL_t(s)
\end{equation*}
\begin{equation*}
= \sum_{m=1}^{\infty}\frac{(-1)^m\theta^m}{m!}\int T_k^m(s)\left\{\int_0^t e^{\langle \eta(z),T(s)\rangle - A(\eta(z))}dA_0(z)\right\}ds
\end{equation*}
\begin{equation*}
=\sum_{m=1}^{\infty}\frac{(-1)^m\theta^m}{m!}\int_0^t\left\{\int T_k^m(s)e^{\langle \eta(z),T(s)\rangle-A(\eta(z))}ds\right\}dA_0(z)
\end{equation*}
\begin{equation}\label{eq5}
=\sum_{m=1}^{\infty}\frac{(-1)^m\theta^m}{m!}\int_0^t  e^{-A(\eta(z))}\frac{\partial^m \big[e^{A(\eta(z))}\big]}{\partial \eta_k^m}dA_0(z)
\end{equation}

\noindent
where the last equality holds by Lemma 1.  Next, we wish to compute the quantities

\begin{equation*}
\mathbb{E}[e^{-\theta T_{k,n}}] = \mathbb{E}\bigg[\prod_{\frac{i}{n}\leq t}\mbox{exp}(-\theta T_{k,n,i})\bigg]=\prod_{\frac{i}{n}\leq t}\mathbb{E}\bigg[\mbox{exp}(-\theta T_{k,n,i})\bigg] 
\end{equation*}

\noindent
and compare their forms to the exponential of \eqref{eq5}. We may achieve this by performing Taylor expansions on the $\mathbb{E}[e^{-\theta T_{k,n,i}}]$ to yield

\begin{equation*}
\mathbb{E}[e^{-\theta T_{k,n,i}}]= 1 + \sum_{m=1}^{\infty}\frac{(-1)^m\theta^m}{m!}\mathbb{E}[T_{k,n,i}^m]
\end{equation*}

 We thus turn to the computation of the moments of the random variables $T_{k,n,i}$, i.e. of the $\mathbb{E}[T_{k,n,i}^m]$. To that end, recall that the random variable $T_{k,n,i}$ is by definition $A_{0,n,i}T_k(S_{n,i})$. Since $S_{n,i}$ is distributed according to the exponential family 
\begin{equation*}
h(s)e^{(\langle \eta_{n,i}(z),T(s)\rangle-A(\eta_{n,i}(z))}
\end{equation*}

\noindent
in order to compute the density of $T_{k,n,i}$, we set $v=T_{k,n,i}(s)$, $\widetilde{T}_j=(T_j\circ T^{-1}_k)(vA_{0,n,i}^{-1}))$ for $j=1,\ldots,m$, and $\psi(vA_{0,n,i}^{-1})=\displaystyle\bigg(\frac{dT_k^{-1}}{dv}\bigg|_{vA_{0,n,i}^{-1}}\bigg)(h\circ T_k^{-1})(vA^{-1}_{0,n,i})$. Then a simple calculation shows that the density of $v$ has the form

\begin{equation*}
\psi(vA_{0,n,i}^{-1})\hspace{.1cm}\mbox{exp}\left(\left(\sum^m_{j=1} \eta_{n,i,j}(T_j\circ T^{-1}_k)(vA_{0,n,i}^{-1})\right) - A(\eta_{n,i}) \right)=
\end{equation*}

\begin{equation*}
 \psi(vA_{0,n,i}^{-1})\hspace{.1cm}\mbox{exp}\left(\eta_{n,i,1}\widetilde{T}_1(vA_{0,n,i}^{-1})+\ldots+\eta_{n,i,k}(vA_{0,n,i}^{-1})+\ldots+\eta_{n,i,m}\widetilde{T}_m(vA_{0,n,i}^{-1}) - A(\eta_{n,i}) \right)=
\end{equation*}

\begin{equation*}
 \psi(vA_{0,n,i}^{-1})\hspace{.1cm}\mbox{exp}\left(\eta_{n,i,1}\widetilde{T}_1(vA_{0,n,i}^{-1})+\ldots+(\eta_{n,i,k}A_{0,n,i}^{-1})v+\ldots+\eta_{n,i,m}\widetilde{T}_m(vA_{0,n,i}^{-1}) - A(\eta_{n,i}) \right).
\end{equation*}
As this integrates to $1$ with respect to $v$, we have

\begin{equation*}
\int \bigg\{\psi(vA_{0,n,i}^{-1})\hspace{.1cm}\mbox{exp}\bigg(\sum^m_{j=1} \eta_{n,i,j}\widetilde{T}_j(vA_{0,n,i}^{-1})\bigg)\bigg\}dv = e^{A(\eta_{n,i})}.
\end{equation*}

\noindent
Therefore  by Lemma 1 we have 

\begin{equation*}
\int \bigg\{v^m\psi(vA_{0,n,i}^{-1})\hspace{.1cm}\mbox{exp}\bigg(\sum^m_{j=1} \eta_{n,i,j}\widetilde{T}_j(vA_{0,n,i}^{-1})\bigg)\bigg\}dv =
\end{equation*}

\begin{equation}\label{eq6}
 \frac{\partial^m}{\partial(\eta_{k,n,i}A_{0,n,i}^{-1})^m}\Big[e^{A(\eta_{n,i})}\Big] = A_{0,n,i}\frac{\partial^m}{\partial(\eta_{k,n,i})^m}\Big[e^{A(\eta_{n,i})}\Big].
\end{equation}

\noindent
Multiplying both sides of \eqref{eq6} by $e^{-A(\eta_{n,i})}$, we conclude that
\begin{equation*}
 \int \bigg\{v^m\psi(vA_{0,n,i}^{-1})\hspace{.1cm}\mbox{exp}\bigg(\sum^m_{j=1} \eta_{n,i,j}\widetilde{T}_j(vA_{0,n,i}^{-1}) - A(\eta_{n,i})\bigg)\bigg\}dv =
\end{equation*}
\begin{equation*}
A_{0,n,i}\bigg(e^{-A(\eta_{n,i})}\frac{\partial^m}{\partial(\eta_{k,n,i})^m}\Big[e^{A(\eta_{n,i})}\Big]\bigg),
\end{equation*}

\noindent
that is,
\begin{equation*}
\mathbb{E}[T_{k,n,i}^m]= A_{0,n,i}\bigg(e^{-A(\eta_{n,i})}\frac{\partial^m}{\partial(\eta_{k,n,i})^m}\Big[e^{A(\eta_{n,i})}\Big]\bigg).
\end{equation*}

\noindent
Thus, 
\begin{equation*}
\mathbb{E}[e^{-\theta T_{k,n,i}}] =  1 + \sum_{m=1}^{\infty}\frac{(-1)^m\theta^m}{m!}\bigg\{A_{0,n,i}\bigg(e^{-A(\eta_{n,i})}\frac{\partial^m}{\partial(\eta_{k,n,i})^m}\Big[e^{A(\eta_{n,i})}\Big]\bigg)\bigg\}.
\end{equation*}

\noindent
Defining 
\begin{equation*}
z_{n,i} =  \sum_{m=1}^{\infty}\frac{(-1)^m\theta^m}{m!}\bigg\{A_{0,n,i}\bigg(e^{-A(\eta_{n,i})}\frac{\partial^m}{\partial(\eta_{k,n,i})^m}\Big[e^{A(\eta_{n,i})}\Big]\bigg)\bigg\}
\end{equation*}

\noindent
we have
\begin{equation*}
\mathbb{E}[e^{-\theta T_{k,n}}] = \prod_{\frac{i}{n}\leq t}\mathbb{E}\big[e^{-\theta T_{k,n,i}}\big]  = \prod_{\frac{i}{n}\leq t}(1 + z_{n,i}),
\end{equation*}

\noindent
and 
\begin{equation*}
 \sum_{\frac{i}{n}\leq t}z_{n,i}= \sum_{m=1}^{\infty}\frac{(-1)^m\theta^m}{m!}\sum_{\frac{i}{n}\leq t}\bigg\{A_{0,n,i}\bigg(e^{-A(\eta_{n,i})}\frac{\partial^m}{\partial(\eta_{k,n,i})^m}\Big[e^{A(\eta_{n,i})}\Big]\bigg)\bigg\}.
\end{equation*}

\noindent
We are interested in the convergence of the above sum of the $z_{n,i}$ as $n \longrightarrow \infty$ as we wish to employ the following lemma due to \cite{Hjort}:

\begin{Lemma}
Let $z_{n,i}$ be real numbers, for $n \geq 1$ and $i \geq 1$. Assume that, as $n \longrightarrow \infty$, (i) $\sum_{a < \frac{i}{n} \leq b} z_{n.i} \longrightarrow z$, (ii) $\max_{a < \frac{i}{n} \leq b}|z_{n,i}| \longrightarrow 0$, (iii) $\lim \sup \sum_{a < \frac{i}{n} \leq b} z_{n.i} \leq M < +\infty$. Then $\prod_{a < \frac{i}{n} \leq b} (1+z_{n.i}) \longrightarrow e^{z}$.
\end{Lemma}

\noindent
It is a straightforward matter to show that sum $z_{n.i}$ as defined satisfy $(ii)$ and $(iii)$ of the above lemma. As for condition $(i)$, note that by an extension of the result found in \cite{Drakakis} to the case of functional derivatives we have
\begin{equation}\label{eq7}
\sum_{\frac{i}{n}\leq t} A_{0,n,i} \bigg(e^{-A(\eta_{n,i})}\frac{\partial \big[e^{A(\eta_{n,i})}\big]}{\partial \eta_{k,n,i}}\bigg)
 \longrightarrow \int_0^te^{-A(\eta(z))}\frac{\partial \big[e^{A(\eta(z))}\big]}{\partial \eta_k(z)}dA_0(z)
\end{equation}

\noindent
as $n \longrightarrow +\infty$. From this fact and \eqref{eq5} we conclude that
\small
\begin{equation*}
\sum_{m=1}^{\infty}\frac{(-1)^m\theta^m}{m!}\sum_{\frac{i}{n}\leq t}\bigg\{A_{0,n,i}\bigg(e^{-A(\eta_{n,i})}\frac{\partial^m}{\partial(\eta_{k,n,i})^m}\Big[e^{A(\eta_{n,i})}\Big]\bigg)\bigg\} \longrightarrow -\int(1-e^{-\theta T_k(s)})dL_t(s).
\end{equation*}
\normalsize
\noindent
i.e.,
\begin{equation*}
 \sum_{\frac{i}{n}\leq t}z_{n,i} \longrightarrow -\int(1-e^{-\theta T_k(s)})dL_t(s).
\end{equation*}

\noindent
We may now invoke Lemma 4 above to obtain 
\begin{equation*}
\mathbb{E}[e^{-\theta T_{k,n}(s)}] \longrightarrow \exp\bigg\{- \int(1-e^{\theta T_k(s)})dL_t(s)\bigg\}
\end{equation*}
as $n \longrightarrow +\infty$.
As in \cite{Hjort}, analogous arguments show that
\begin{equation*}
\mathbb{E}\Big[\mbox{exp}\Big(-\sum^j_{i=1}\theta_iT_{k,n}(t_{i-1},t_i]\Big)\Big] \longrightarrow \exp\bigg\{- \sum^j_{i=1 }\int(1-e^{\theta T_k(s)})dL_{(t_{i-1},t_i]}(s)\bigg\},
\end{equation*}
\noindent
where $\theta_i > 0 \hspace{.1cm} \forall i=1,\ldots,j$. As previously stated, by the Cram\'{e}r-Wold device and Theorem 16.16 in \cite{Kallenberg}, the finite dimensional distributions of $\{T_{k,n}(s)\}$ converge properly. Finally, the fact that the sequence $\{T_{k,n}(s)\}_{n=1}^{\infty}$ is \emph{tight} in the space of all functions that are right continuous with left hand limits in the Skorohod topology follows from \eqref{eq1} and the proof of 15.6 in \cite{Billingsley}. The convergence of the finite dimensional distributions of $\{T_{k,n}(s)\}$ plus tightness of the sequence implies the existence of the process $T(t)$ as promised. Hence the proof of the theorem is complete.
\end{proof}

Condition 1 is required so that an explicit form of the density function of $T_k$ can be found, which in turn permits the computation of the the transform $\mathbb{E}[\exp(-\theta T_k)]$ linking the random variable $T_k$ to the L\'{e}vy measure

\begin{equation*}
 \nu(du,dA_0(z)) =  \exp\big(\langle\eta(z),u\rangle-A(\eta(z))\big)\frac{dT_k^{-1}}{du}dA_0(z)du.
\end{equation*}

As noted in section 2.1 this linkage between $T_k$ and $ \nu(du,dA_0(z))$ completely establishes the distributional properties of the random variable via a Poisson process with intensity measure given by the L\'{e}vy measure.

Condition 2 of the theorem simply requires that $\eta(z)$ does in fact determine a well defined exponential family. Condition 3 is a technical requirement which is directly tied to the construction of the L\'{e}vy process $T(t)$. The condition, loosely interpreted, states that the natural parameter space is closed under \emph{contraction towards} $0$, i.e. if one takes any point in the natural parameter space, and shrinks it in absolute value by an amount $\varepsilon$, then the resulting value is still in the natural parameter space. Note that this is not quite equivalent to the well known property of convexity of the natural parameter space found in \cite{TSH}, as the element $0$ need not be in the space. This is precisely the case for a beta distributed random variable. Finally, we note that if there are multiple sufficient statistics satisfying condition 1 of the theorem, then the L\'{e}vy measures resulting from different choices of $T_k(s)$ will all be absolutely continuous with respect to one another, i.e. all measures will be equivalent.

The reader is now reminded of the claim from the end of section 3.3 that our construction is capable of yielding stochastic processes which are not compound Poisson processes. To illustrate this fact, consider the Pareto exponential family with density given by  $p(u|\alpha) = \frac{\alpha u_m^{\alpha}}{u^{\alpha + 1}}$, take $\alpha(z)=z$, and $A_0(z)$ to be standard Lebesgue measure on $(0,\infty)$. Then 
\begin{equation*}
\mathbb{E}[\exp(-\theta T)]=\exp\left\{-\int_{(0,\infty)}(1-e^{-\theta u})dL_t(u)\right\} =
\end{equation*}
\begin{equation*}
\exp \left\{-\int_{(0,\infty)}(1-e^{-\theta u})\bigg(\int_0^t \exp\big(\langle \eta(z),U\rangle -A(\eta(z))\big)\frac{dT_k^{-1}}{du}dA_0(z)\bigg)du\right\}=
\end{equation*}
\begin{equation*}
\exp \left\{-\int_{(0,\infty)}(1-e^{-\theta u})\frac{dT_k^{-1}}{du}\bigg(\int_0^t \exp\big(\langle \eta(z),U\rangle -A(\eta(z))\big)dA_0(z)\bigg)du\right\}=
\end{equation*}
\begin{equation*}
\exp \left\{-\int_{(0,\infty)}(1-e^{-\theta u})\frac{1}{u}\bigg(\int_0^t z\exp\big(z(u+\mbox{ln}(u_m))\big)dz\bigg)du\right\} = 
\end{equation*}
\begin{equation*}
\exp \left\{\int_{(0,\infty)} e^t(t-1)(e^{-\theta u}-1)\bigg(\frac{1}{u^2+u\mbox{ln}(u_m)}du\bigg)\right\}.
\end{equation*}
\noindent
Upon comparison of the last term in the string of equalities above with the form of $\mathbb{E}[\exp(-\theta T)]$ in Lemma 3 from section 3.2, we see that we cannot find a density $\sigma(u)$ and a positive constant $c$ such that 
\begin{equation*}
\exp \left\{\int_{(0,\infty)} e^t(t-1)(e^{-\theta u}-1)\bigg(\frac{1}{u^2+u\mbox{ln}(u_m)}du\bigg)\right\} = \mbox{exp}\Big \{tc \displaystyle\int_{(0,\infty)}(e^{-\theta u}-1)\sigma(du)\Big \}
\end{equation*}
\noindent
for $\theta \geq 0$, i.e., the process $T(t)$ is \emph{not} a compound Poisson process.

\section{Explicit formulae for posterior density parameters}
Section four of this paper begins with a consideration of the settings for and assumptions utilized in proofs of presently known results regarding conjugacy of exponential families in an infinite dimensional setting. We then point to the advantages of casting the issue of conjugacy purely in terms of conditional expectation operators. Finally, we state the main theorem of the section and interpret the various regularity conditions set forth in the theorem.

\subsection{Presently known results and proof techniques and assumptions}

The first known proof of conjugacy for a pair of exponential family completely random measures was given in \cite{Hjort2} and leads directly to the conjugacy of the beta and Bernoulli processes. We note here that although the original result appears in \cite{Hjort2}, credit for this achievement is usually ascribed to \cite{Hjort} and a related result (whose proof requires six pages of hard integrations) from which the beta-Bernoulli conjugacy may be derived appears in \cite{Kim}. Hjort's methodology for the proof of his result was to directly compute $\frac{\mathcal{P}_1(G)}{F_0(x,\infty)}$ and then use the resulting probability measure to evaluate the Laplace transforms of the beta process.  In this expression $\mathcal{P}_1=\mathcal{P}(X > x, A \in G)$, where $\mathcal{P}$ is a probability distribution over $([0,\infty)^n \times \mathcal{B}, \mathcal{C}_n \times \Sigma_{\mathcal{B}})$ where $\mathcal{C}_n$ denotes the Borel $\sigma$-algebra on $[0,\infty)^n$, $\Sigma_{\mathcal{B}}$ is the $\sigma$-algebra generated by the Borel cylinder sets on $\mathcal{B}$, which is the set of all right-continuous, nondecreasing functions on  $[0,\infty)$ with $B(0)=0$, the space $[0,\infty)^n$ houses the $i.i.d.$ observations $X_1,\ldots,X_n$, and $F_0(x) = E(F(x))$, $F(x)$ being the cumulative hazard rate function. The computation of the quantity in question is in fact quite laborious, relies heavily on specific properties possessed by the hazard rate function, and is only achieved through what Hjort himself terms ``\emph{heroic integrations}.'' \cite{Hjort,Hjort2}

While the above conjugacy result firmly assumed a preeminent position in the machine learning literature due to its defining role in the Indian buffet process as illustrated in \cite{Thibaux}, similar and equally pertinent results were developed. An excellent survey of such results can be found in \cite{Broderick}. Indeed, the output achieved in \cite{Broderick} not only contains a thorough overview of conjugacy in completely random measures, it also develops a framework from which new conjugacy results were proved. For the sake of elucidation of the scope of the results contained in the present work, and explication of the conditions required to produce them, it is essential to observe the assumptions required to obtain the outcomes in \cite{Broderick}.

The setting in \cite{Broderick} is that of exponential family completely random measures. Note that we adhere to the symbology established in \cite{Broderick} which differs somewhat from that established in section 2.3 of the present work since the translation between the two is trivial. While the authors establish a general framework for proving conjugacy in completely random measures, a number of assumptions apply to their setting. The first of these, is that the rate measure, $\mu(d\theta \times d\psi)$, which characterizes the Poisson process associated with the completely random measure is \emph{assumed} to be factorizable as $\mu(d\theta \times d\psi)=\nu(d\theta)\cdot G(d\psi)$ where $\nu$ is a $\sigma$-finite deterministic measure on $\mathbb{R}_{+}$ and $G$ is a proper distribution on $\Psi$, what they term to be \text{``a space of traits.''}

Secondly, they place the requirement on the fixed component of the completely random measure, $\Phi_f$, that the number of fixed atoms is finite, and from this assumption it must follow that $\nu(\mathbb{R}_{+}) = \infty$. Third, denoting the distribution governing the observed traits by $H(dx|\theta)$, the condition is imposed that $H(dx|\theta)$ is a \emph{discrete} distribution. Regarding this condition, the authors observe that given the framework in the paper 

\begin{quote}
$H(dx|\theta)$ cannot be purely continuous for all $\theta$. Though this line of reasoning does not necessarily preclude a mixed continuous and discrete $H$, we henceforth assume that $H(dx|\theta)$ is discrete with support $\mathbb{Z}_n=\{1,2,\ldots\}$, for all $\theta$.  
\end{quote}

Finally, in light of the above, the authors adopt the notation $h(x|\theta)$ for the probability mass function of $x$ given $\theta$, and note one more restriction as a consequence of the antecedent assumptions. Specifically, expressing the rate measure of the de rigueur marked Poisson point process as $\mu(d\theta \times dx) := \nu(d\theta)h(x|\theta)$  and the rate measure of the thinned Poisson point process as $\nu_x(d\theta):= \nu(d\theta)h(x|\theta)$ finiteness considerations on the number of atoms require $\sum_{x=1}^{\infty}\nu_x(\mathbb{R}_{+}) < \infty$.

These assumptions prove to be fruitful in the sense that they allow the authors to obtain all previously known results regarding conjugacy of completely random measures as well as prove conjugacy in the infinite dimensional case for the gamma-Poisson processes, a result which, at that time, was not publicized. 

By contrast, employing the machinery available in the theory of conditional expectation operators, we are able to prove conjugacy for a wider class of exponential family completely random measures that include, but do not require, discrete distributions for the likelihood processes. We are able to achieve these results by defining a canonical construction for stochastic processes that take values in $\mathbb{R}_{+}$ and whose L\'{e}vy measure densities belong to an exponential family. Furthermore, we only require mild regularity conditions of the type required to prove familiar results from measure theory, such as existence theorems for stochastic processes possessing prescribed finite dimensional distributions, uniform absolute continuity for families of measures, as well as weak convergence results in spaces of measures. We embark on a detailed discussion of our main theorem and such regularity conditions as required for its proof in the next section.

\subsection{Statement and analysis of the main theorem}

We now state and prove our main result for this section. The statement of Theorem 4 warrants a number of comments regarding the setting of the canonical construction, the technical conditions necessitated, and the conjugate objects it produces. As such we recall the setup and technical assumptions of Theorem 3 in section 3.

We start with a given $l$-dimensional exponential family, $p(x|\eta)$, defined on $[0,\infty)$, for which one has, as before, not a vector of scalars supplying the natural parameters, but rather a function $\eta(z)=(\eta_1(z),\ldots,\eta_l(z))$ which supplies the natural parameters. The sufficient statistics of the exponential family are given by $T(s) = (T_1(s),\ldots,T_l(s))$, and the canonical construction achieved in section 3 employs only \emph{one} of the sufficient statistics, say $T_k(s)$ for a fixed $k$ where $1 \leq k \leq m$. The chosen $T_k(s)$ and the given function $\eta(z)$ satisfy the regularity conditions in the existence theorem in section 3. Hence for any positive, increasing $A_0(z)$ on $[0,\infty)$ which is right continuous with left hand limits, the construction of the theorem results in a L\'{e}vy process $T(t)$ and a corresponding L\'{e}vy measure whose density is given by the specified form of the exponential family.

The L\'{e}vy process $T(t)$ is constructed by first, letting $\{S_{n,i}\}_{n,i\in \mathbb{N}^2}$ be a collection independent random variables, of which each $S_{n,i}$ is distributed according to the exponential family $p(s|\eta_{n,i})$, where  $\eta_{n,i}=\eta\big(\frac{i-\frac{1}{2}}{n}\big)$. Second, employing the $k$th sufficient statistic of the exponential family, $T_k$, we define $A_{0,n,i}=A_0(\frac{i-1}{n},\frac{i}{n}]$ and $T_{k,n,i}=A_{0,n,i}T_k(S_{n,i})$. This yields $\{T_{k,n,i}\}_{n,i\in \mathbb{N}^2}$ which is a collection of independent random variables distributed according to $p(T^{-1}_{k,n,i}|\eta_{n,i})\frac{dT_k^{-1}}{du}$ where $T_{k,n,i}(s)=u$. Finally, the sums
\begin{center}
$T_{k,n}(t)=\displaystyle\sum_{\frac{i}{n}\leq t}T_{k,n,i}$
\end{center}

\noindent
converge as $n \longrightarrow \infty$ to the desired stochastic process, i.e. completely random measure, whose L\'{e}vy measure density is given by the chosen exponential family. It is worth noting that although the quantity $T_k(t)$ at first glance appears to be able to be factored out of the summation above as it does not seem to depend on $i$ and $n$, this is in fact \emph{not} the case since the underlying distribution of $T_k(t)$ \emph{does} depend on $i$ and $n$ via the exponential family $p_{i,n}(x|\eta)$.

In proving the existence of triples of stochastic processes. i.e. prior, likelihood, and posterior, whose L\'{e}vy measure densities are from positive valued exponential families which are conjugate from the perspective of conditional expectation operators,  our methodology is as follows: We start with a likelihood-prior pair of exponential family distributions in a \emph{finite dimensional} setting along with an appropriate $\eta(z)$ and $A_0(z)$  satisfying the conditions of the main theorem in section 3, and from this obtain the existence of the prior processes.

The construction of the likelihood process is a generalization of the construction of the Bernoulli process with base measure $B$, a beta process CRM, found in \cite{Thibaux}. Given a prior process as constructed in Theorem 3, say $T(t)$, one invokes the CRM representation from \cite{Kingman67} discussed in section 2.4 to realize the process $T(t)$ as
\begin{equation}\label{T}
T = \sum_{\lambda \in \Lambda}\omega_{\lambda}\delta_{\theta_{\lambda}}
\end{equation}
\noindent
where $\{(\omega_{\lambda},\theta_{\lambda}\}_{\lambda \in \Lambda}$ is a draw from a non-homogeneous Poisson process, and $\delta_{\theta_{\lambda}}$ is an atom at $\theta_{\lambda}$ with weight $\omega_{\lambda}$ in $T(t)$. We then define the likelihood process, LP$(T)$, with base measure $T(t)$ as a stochastic process whose realizations are of the form
\begin{equation}\label{L}
\mbox{LP}(T) = \sum_{\lambda \in \Lambda}\gamma_{\lambda}\delta_{\theta_{\lambda}}
\end{equation}
\noindent
where the atoms of LP$(T)$ are the atoms in the representation of $T$, and the weights $\gamma_{\lambda}$ are draws from an exponential family likelihood parametrically conjugate to the chosen prior, and the exponential family parameter governing draw $\gamma_{\lambda}$ is $\omega_{\lambda}$.    

Next, given observations $X_1,\ldots X_n$ of the likelihood processes, we then invoke the main theorem from section 3 a second time to prove the existence of the ``posterior'' process. Note that at this stage in the proof the aforementioned process is only a \emph{candidate} for the prior, as we have yet to argue that this process is derived from the conditional probability of the prior \emph{given} the observed data from the likelihood process. The task of demonstrating that these triples do in fact form the prior, likelihood, and posterior in a conjugate family when viewed in the setting of conditional expectation operators, is completed by calculating the finite dimensional distributions of the conditional Laplace transform of the prior process. In other words we compute $E_{P(dT|X=X_1)}[\mbox{exp}(-\theta T(s)]$ for all $s$ and from these computations assemble the finite dimensional distributions of the posterior process.

That this procedure outline above should be considered natural is witnessed by the observation that the known conjugacy results in nonparametric settings do in fact mirror the results in the parametric, i.e. finite dimensional, cases. For example, if one recalls the statement of conjugacy in the case of the Bernoulli and beta processes, $X$ and $B$ respectively, this is expressed as 
\begin{equation*}
 B|X_1,\ldots,X_n \sim \mbox{BP}\left(c+n, \frac{c}{c+n}B_0 + \frac{1}{c+n}\sum X_i\right)
\end{equation*} 
 
\noindent
where $c$ is the concentration function, and $B_0$ is the initial measure. Now, recall the classical parametric form of conjugacy between the Bernoulli and beta distributions. If $\alpha$ and $\beta$ are the parameters governing the beta distribution, then given $n$ observations of a Bernoulli random variable, the parameters of the posterior beta distribution are given by 
\begin{equation*}
\alpha_{post} = \alpha + \sum^n_{i=1}x_i,\hspace{.3cm}\beta_{post} = \beta + n - \sum^n_{i=1}x_i
\end{equation*}

\noindent
Recalling that a beta process is taken to be a stochastic process whose infinitesimal increments are  BP$(cB_0,c(1-B_0))$ distributed, i.e. $\alpha = cB_0$ and $\beta = c(1-B_0)$, we see that the parameters of the posterior beta process are given by 
\begin{equation*}
(c+n)\Big(\frac{c}{c+n}B_0 + \frac{1}{c+n}\sum X_i\Big)=cB_0 + \sum X_i = \alpha + \sum X_i
\end{equation*}

\noindent
and 
\begin{equation*}
(c+n)\Big(1-\Big(\frac{c}{c+n}B_0 + \frac{1}{c+n}\sum X_i\Big)\Big)=c(1 - B_0) + n - \sum X_i = \beta + n - \sum X_i.
\end{equation*}

\noindent
Hence, the form of the posterior parameters in the infinite dimensional case directly mirror that of the form of the posterior parameters in the finite dimensional case.

We now state and prove Theorem 2.

\begin{Theorem}
Let $T_k(x)$ be a sufficient statistic of an $l$-dimensional exponential family $p(x|\eta)$, $\eta(z)=(\eta_1(z),\ldots,\eta_l(z))$ a vector of piecewise continuous, nonnegative functions on $(0,\infty)$, $A_0(z)$ a positive, increasing function on $[0,\infty)$ which is right continuous with left hand limits, and $p(y|x)$ the distribution of a likelihood parametrically conjugate to the prior $p(x|\eta)$.  Assume the following conditions hold:

\begin{enumerate} 
\item $T_k^{-1}(u)$ exists and is differentiable;
\item for all $z\hspace{.1cm} \in \hspace{.1cm} (0,\infty) \hspace{.1cm}$ we have $\eta(z) \in \Xi\hspace{.1cm}$, the natural parameter space of $p(x|\eta)$;
\item if $(\eta_1(z),\ldots,\eta_l(z)) \in \Xi$\hspace{.2cm}then for every $ 0<\varepsilon <1$ \hspace{.1cm}it follows that \newline$(\eta_1(z),\ldots,\varepsilon \eta_k(z),\ldots,\eta_l(z)) \in \Xi$.
\end{enumerate}
Additionally, let $(Y_t)$ be a likelihood process of the form \eqref{L} with base measure $T$ of the form \eqref{T} whose existence is guaranteed by Theorem 3. Then the stochastic processes $(Y_t)$ and $(T_t)$ form an infinite dimensional likelihood-prior conjugate pair where 

\begin{equation*}
\mathbb{E}[\exp(-\theta T)|\hspace{.1cm}Y=Y_1,\ldots Y_n] = 
\end{equation*} 
\begin{equation*}
\exp\left\{-\int(1-e^{-\theta u})dL_t(u)\right\} = \exp\left\{-\int(1-e^{-\theta T_k(s)})dL_t(s)\right\},
\end{equation*}
\noindent such that 
\small
\begin{equation*}
dL_t(u) = \left\{\int_0^t \exp\big(\langle \tau (\eta(z),Y_1\ldots Y_n),U\rangle-A( \tau (\eta(z),Y_1\ldots Y_n))\big)\frac{dT_k^{-1}}{du}dA_0(z)\right\}du
\end{equation*}\\[.25cm]
\normalsize
and
\begin{equation*}
dL_t(s) = \left\{\int_0^t \exp\big(\langle \tau (\eta(z),Y_1\ldots Y_n),T(s)\rangle-A( \tau (\eta(z),Y_1\ldots Y_n))\big)dA_0(z)\right\}ds,
\end{equation*}\\[.25cm]
where $U=(T_1(T_k^{-1}(u)),...,T_l(T_k^{-1}(u)))$, $u=T_k(s)$, and $\tau$ is the function governing the form of the posterior parameters derived from the prior and observations for the parametric likelihood-prior pair $p(y|x)$, $p(x|\eta)$.\\[.25cm]
\end{Theorem}

\emph{Proof}:\hspace{.2cm}The majority of this proof will be dedicated to showing that the \emph{conditional} Laplace transforms of the sums $\sum^k_{i=1}\theta_iT_{k,n}(t_{i-1},t_i]$, $\theta_i > 0 \hspace{.1cm} \forall i$, converge as $n \longrightarrow \infty$ to the quantity 

\begin{equation*}
\exp\bigg\{- \sum^k_{i=1 }\int(1-e^{\theta T_k(s)})dL_{(t_{i-1},t_i]}(s)\bigg\}
\end{equation*}
\noindent
where $dL_T(s)$ has the form stated in the current theorem. To that end performing a Taylor expansion on $e^{-\theta T_k(s)}$ yields
\begin{equation*}
-\int(1-e^{-\theta T_k(s)})dL_t(s) = -\int\bigg(1-\bigg(\sum_{m=0}^{\infty}\frac{(-1)^m\theta^mT_k^m(s)}{m!}\bigg)\bigg)dL_t(s)
\end{equation*}
\begin{equation*}
=\sum_{m=1}^{\infty}\frac{(-1)^m\theta^m}{m!}\int T_k^m(s)dL_t(s)
\end{equation*}
\begin{equation*}
= \sum_{m=1}^{\infty}\frac{(-1)^m\theta^m}{m!}\int T_k^m(s)\left\{\int_0^t e^{\langle \tau(\eta(z),Y),T(s)\rangle - A(\tau(\eta(z),Y))}dA_0(z)\right\}ds
\end{equation*}
\begin{equation*}
=\sum_{m=1}^{\infty}\frac{(-1)^m\theta^m}{m!}\int_0^t\left\{\int T_k^m(s)e^{\langle \tau(\eta(z),Y),T(s)\rangle-A(\tau(\eta(z),Y))}ds\right\}dA_0(z)
\end{equation*}
\begin{equation}\label{eq5}
=\sum_{m=1}^{\infty}\frac{(-1)^m\theta^m}{m!}\int_0^t  e^{-A(\tau(\eta(z),Y))}\frac{\partial^m \big[e^{A(\tau(\eta(z),Y))}\big]}{\partial (\tau(\eta_k,Y))^m}dA_0(z)
\end{equation}

\noindent
Next we compute the quantities

\begin{equation*}
E_{P(dT|Y)}[\exp(-\theta T_{k,n}(s)]=\mathbb{E}[e^{-\theta T_{k,n}}|Y] =
\end{equation*}
\begin{equation*}
 \mathbb{E}\bigg[\prod_{\frac{i}{n}\leq t}\exp(-\theta T_{k,n,i})|Y\bigg]=\prod_{\frac{i}{n}\leq t}\mathbb{E}[\exp(-\theta T_{k,n,i})|Y] 
\end{equation*}
\noindent
and compare their forms to the exponential of \eqref{eq5}. We may achieve this by performing Taylor expansions on the $\mathbb{E}[e^{-\theta T_{k,n,i}}|Y]$ to yield

\begin{equation*}
\mathbb{E}[e^{-\theta T_{k,n,i}}|Y]= 1 + \sum_{m=1}^{\infty}\frac{(-1)^m\theta^m}{m!}\mathbb{E}[T_{k,n,i}^m|Y]
\end{equation*}
\noindent
and thus we turn to the computation of the moments of the random variables $T_{k,n,i}$, i.e. of the $\mathbb{E}[T_{k,n,i}^m|Y]$. To that end, recall that the random variable $T_{k,n,i}$ is by definition $A_{0,n,i}T_k(S_{n,i})$, and $S_{n,i}$ is distributed according to the exponential family 
\begin{equation*}
h(s)e^{(\langle \eta_{n,i}(z),T(s)\rangle-A(\eta_{n,i}(z))}
\end{equation*}

\noindent
which is by assumption parametrically conjugate to the random variable $Y_{n,i}$ and hence has a conditional density of the form

\begin{equation*}
h(s)e^{(\langle \tau(\eta_{n,i}(z),Y),T(s)\rangle-A(\tau(\eta_{n,i},Y)(z))}.
\end{equation*}

 Thus, in order to compute the density of $T_{k,n,i}$ conditional on $Y$, we set $v=T_{k,n,i}(s)$, $\widetilde{T}_j=(T_j\circ T^{-1}_k)(vA_{0,n,i}^{-1}))$ for $j=1,\ldots,m$, and $\psi(vA_{0,n,i}^{-1})=\displaystyle\bigg(\frac{dT_k^{-1}}{dv}\bigg|_{vA_{0,n,i}^{-1}}\bigg)(h\circ T_k^{-1})(vA^{-1}_{0,n,i})$. Then a simple calculation shows that the density of $v$ has the form

\begin{equation*}
\psi(vA_{0,n,i}^{-1})\hspace{.1cm}\exp\left(\left(\sum^m_{j=1} \tau(\eta_{n,i,j},Y)(T_j\circ T^{-1}_k)(vA_{0,n,i}^{-1})\right) - A(\tau(\eta_{n,i},Y)) \right)=
\end{equation*}

\begin{equation*}
 \psi(vA_{0,n,i}^{-1})\hspace{.1cm}\exp(\tau(\eta_{n,i,1},Y)\widetilde{T}_1(vA_{0,n,i}^{-1})+\ldots+\tau(\eta_{n,i,k},Y)(vA_{0,n,i}^{-1})+\ldots
\end{equation*}
\begin{equation*}
\hspace{4cm}\ldots+\tau(\eta_{n,i,m},Y)\widetilde{T}_m(vA_{0,n,i}^{-1}) - A(\tau(\eta_{n,i},Y)) )=
\end{equation*}
\begin{equation*}
 \psi(vA_{0,n,i}^{-1})\hspace{.1cm}\exp (\tau(\eta_{n,i,1},Y)\widetilde{T}_1(vA_{0,n,i}^{-1})+\ldots+(\tau(\eta_{n,i,k},Y)A_{0,n,i}^{-1})v+\ldots
\end{equation*}
\begin{equation*}
\hspace{4cm}\ldots+\tau(\eta_{n,i,m},Y)\widetilde{T}_m(vA_{0,n,i}^{-1}) - A(\tau(\eta_{n,i},Y))).
\end{equation*}
As this integrates to $1$ with respect to $v$, we have

\begin{equation*}
\int \bigg\{\psi(vA_{0,n,i}^{-1})\hspace{.1cm}\exp\bigg(\sum^m_{j=1} \tau(\eta_{n,i,j},Y)\widetilde{T}_j(vA_{0,n,i}^{-1})\bigg)\bigg\}dv = e^{A(\tau(\eta_{n,i},Y))}.
\end{equation*}

\noindent
Once again by Lemma 1 we have 

\begin{equation*}
\int \bigg\{v^m\psi(vA_{0,n,i}^{-1})\hspace{.1cm}\exp\bigg(\sum^m_{j=1} \eta_{n,i,j}\widetilde{T}_j(vA_{0,n,i}^{-1})\bigg)\bigg\}dv =
\end{equation*}

\begin{equation}\label{eq6}
 \frac{\partial^m}{\partial(\tau(\eta_{k,n,i},Y)A_{0,n,i}^{-1})^m}\Big[e^{A(\tau(\eta_{n,i},Y))}\Big] = A_{0,n,i}\frac{\partial^m}{\partial(\tau(\eta_{k,n,i},Y))^m}\Big[e^{A(\tau(\eta_{n,i},Y))}\Big].
\end{equation}

\noindent
Multiplying both sides of \eqref{eq6} by $e^{-A(\tau(\eta_{n,i},Y))}$, we conclude that
\begin{equation*}
 \int \bigg\{v^m\psi(vA_{0,n,i}^{-1})\hspace{.1cm}\exp\bigg(\sum^m_{j=1} \tau(\eta_{n,i,j},Y)\widetilde{T}_j(vA_{0,n,i}^{-1}) - A(\tau(\eta_{n,i},Y))\bigg)\bigg\}dv
\end{equation*}

\begin{equation*}
 = A_{0,n,i}\bigg(e^{-A(\tau(\eta_{n,i},Y))}\frac{\partial^m}{\partial(\tau(\eta_{k,n,i},Y))^m}\Big[e^{A(\tau(\eta_{n,i},Y))}\Big]\bigg),
\end{equation*}
\noindent
that is,
\begin{equation*}
\mathbb{E}[T_{k,n,i}^m|Y]= A_{0,n,i}\bigg(e^{-A(\tau(\eta_{n,i},Y))}\frac{\partial^m}{\partial(\tau(\eta_{k,n,i},Y))^m}\Big[e^{A(\tau(\eta_{n,i},Y))}\Big]\bigg).
\end{equation*}

\noindent
Thus, 
\begin{equation*}
\mathbb{E}[e^{-\theta T_{k,n,i}}|Y] =  1 + \sum_{m=1}^{\infty}\frac{(-1)^m\theta^m}{m!}\bigg\{A_{0,n,i}\bigg(e^{-A(\tau(\eta_{n,i},Y))}\frac{\partial^m}{\partial(\tau(\eta_{k,n,i},Y))^m}\Big[e^{A(\tau(\eta_{n,i},Y))}\Big]\bigg)\bigg\}.
\end{equation*}

\noindent
As before, defining 
\begin{equation*}
z_{n,i} =  \sum_{m=1}^{\infty}\frac{(-1)^m\theta^m}{m!}\bigg\{A_{0,n,i}\bigg(e^{-A(\tau(\eta_{n,i},Y))}\frac{\partial^m}{\partial(\eta_{k,n,i})^m}\Big[e^{A(\tau(\eta_{n,i},Y))}\Big]\bigg)\bigg\}
\end{equation*}

\noindent
we have
\begin{equation*}
\mathbb{E}[e^{-\theta T_{k,n}}|Y] = \prod_{\frac{i}{n}\leq t}\mathbb{E}\big[e^{-\theta T_{k,n,i}}|Y\big]  = \prod_{\frac{i}{n}\leq t}(1 + z_{n,i}),
\end{equation*}

\noindent
and 
\begin{equation*}
 \sum_{\frac{i}{n}\leq t}z_{n,i}= \sum_{m=1}^{\infty}\frac{(-1)^m\theta^m}{m!}\sum_{\frac{i}{n}\leq t}\bigg\{A_{0,n,i}\bigg(e^{-A(\tau(\eta_{n,i},Y))}\frac{\partial^m}{\partial(\tau(\eta_{k,n,i},Y))^m}\Big[e^{A(\tau(\eta_{n,i},Y))}\Big]\bigg)\bigg\}.
\end{equation*}

\noindent
We are interested in the convergence of the above sum of the $z_{n,i}$ as $n \longrightarrow \infty$ and using arguments parallel to those in the proof of Theorem 1 we can show

\noindent
\begin{equation*}
 \sum_{\frac{i}{n}\leq t}z_{n,i} \longrightarrow -\int(1-e^{-\theta T_k(s)})dL_t(s).
\end{equation*}

\noindent
We may now invoke Lemma 4 to obtain 
\begin{equation*}
\mathbb{E}[e^{-\theta T_{k,n}(s)}|Y] \longrightarrow \exp\bigg\{- \int(1-e^{\theta T_k(s)})dL_t(s)\bigg\}
\end{equation*}
as $n \longrightarrow +\infty$, hence
\begin{equation*}
\mathbb{E}\Big[\exp\Big(-\sum^j_{i=1}\theta_iT_{k,n}(t_{i-1},t_i]\Big)\big|Y\Big] \longrightarrow \exp\bigg\{- \sum^j_{i=1 }\int(1-e^{\theta T_k(s)})dL_{(t_{i-1},t_i]}(s)\bigg\},
\end{equation*}
\noindent
where $\theta_i > 0 \hspace{.1cm} \forall i=1,\ldots,j$. Once again, by the Cram\'{e}r-Wold device and Theorem 16.16 in \cite{Kallenberg}, the finite dimensional distributions of $\{T_{k,n}(s)|Y\}$ converge properly. By a result of \cite{Choksi} the conditional finite dimensional distributions of the process determine the conditional distribution of the process, hence the proof of the theorem is complete.

\section{Examples}

In this section we provide examples of the applications of Theorem 1 and Theorem 2. First, with regards, to Theorem 1, we demonstrate that the L\'{e}vy measures for the beta and gamma processes are obtainable from the L\'{e}vy measure representation in \eqref{eq2}.  We follow up with the computation of a L\'{e}vy measure for a process with infinitesimally Pareto distributed increments.  

The methodology employed in the examples illustrating the use of Theorem 1 stands in contrast to the L\'{e}vy measure decomposition procedure demonstrated in \cite{Wang}. This decomposition procedure requires one to have the L\'{e}vy measure for the completely random measure in hand, and then after examining the particular form of the measure, apply a number of series expansions and identities particular to the moments of the process to arrive at the decomposition. Thus, starting with a completely random measure, the researcher arrives at a decomposition which, while permitting simulation of the process, does not allow the researcher to specify the L\'{e}vy measures which are used to generate the simulation. \cite{Wang} acknowledge this in by noting that all that can be said about the components of the resulting decomposition are that they are L\'{e}vy processes. In comparison, employing the construction in Theorem 1, our procedure allows the researcher to specify \emph{a priori} the specific forms of the L\'{e}vy measures, prove the existence of the corresponding processes, and then arrive at the desired completely random measure via an infinite sum.

We continue this section by demonstrating the simplicity with which one can compute prior parameters when employing Theorem 2. This computation is performed for two nonparametric conjugate pairs. The first consists of a log-normal process for the likelihood, and a generalized gamma process for the posterior, while the second employs a Pareto process for the likelihood and a generalized gamma process for the posterior.  

\subsection{The beta process}

In \cite{Wang} the beta process $B$ with corresponding L\'{e}vy measure given by  $\nu(ds,dz) = c(z)s^{-1}(1-s)^{c(z) - 1}dsd\mu(z)$ is decomposed into infinite sums $B=\sum_n B_n$ and $\nu = \sum_n \nu_n$ where the $B_n$ are L\'{e}vy processes with  L\'{e}vy measures $\nu_n(ds,dz)$, given by

\begin{equation}\label{eq8}
\nu_n(ds,dz) = \mbox{Beta}(1,c(z)+n)ds\frac{c(z)}{c(z)+n}d\mu(z).
\end{equation}

The function $c(z)$ in the above is assumed to be a piecewise continuous positive function on $(0,\infty)$. The decomposition of the L\'{e}vy measure $\nu$ into an infinite series with components given by \eqref{eq5} is achieved by simply writing the $s^{-1}$ term in $\nu$ as an infinite series and then distributing the remaining terms, as well as employing identities for the gamma function.

Rather than beginning with the measure $\nu$ and then subsequently arriving at the form of the decomposition, the L\'{e}vy measures in \eqref{eq8} can be constructed directly from Theorem 1, hence we will verify that the conditions of Theorem 4 are satisfied. 

First, the sufficient statistics for the beta distribution are $T_1(x) = \ln(x)$ and $T_2(x) = \ln(1-x)$ and both of these functions has an infinitely differentiable inverse, hence condition 1 of the theorem is satisfied.  Second,  defining $\eta(z)=(1,c(z)+n)$ yields a function satisfying condition 2 of the theorem. Finally, since the natural parameter space $\hspace{.05cm}\Xi$ of the beta distribution consists of all $(\alpha,\beta)$ where $\alpha$,$\beta > 0$ we see that for any $\varepsilon > 0$ if $(\alpha,\beta)$ $\in$ $\Xi$ then both  $(\varepsilon \alpha,\beta)$ and  $(\alpha,\varepsilon \beta)$ are in $\Xi$. Thus condition 3 is satisfied. As all the requisite conditions are satisfied, we can construct processes $B_n$ for $n=1,\ldots$ whose L\'{e}vy measures have the form in \eqref{eq7}. This is achievable by Theorem 1 as each $\nu_n$ is of the form in \eqref{eq8} with 
$A_0(z)=(c(z) / (c(z)+n)) F(z),$
where $F(z)$ is the cumulative distribution of $\mu$, resulting in an $A_0(z)$ which is a positive, increasing function which is right continuous with left hand limits.

\subsection{The gamma process}

Similar to the decomposition of the L\'{e}vy measure for the beta process, in~\cite{Wang} a decomposition for the gamma process L\'{e}vy measure
 
\begin{equation}\label{eq9}
\nu(ds,dz) = s^{-1}\exp\bigg(\frac{-s}{c(z)}\bigg)dsd\mu(z)
\end{equation}
is obtained. In this case the L\'{e}vy measure is decomposed into a doubly indexed infinite sum of L\'{e}vy measures $\nu_{k,h}(ds,dz)$,  where for every $k,h$

\begin{equation}\label{eq10}
\nu_{k,h}(ds,dz) = \mbox{Gamma}\bigg(h,\frac{c(z)}{k+1}\bigg)ds\frac{d\alpha(z)}{(k+1)^hh}.
\end{equation}
The decomposition of the L\'{e}vy measure in \eqref{eq9} into a doubly indexed infinite series with components given by \eqref{eq10} is achieved not by a simple series expansion of a term in \eqref{eq9}, but rather by an iterative procedure of rewriting  the exponential term in the measure as a product of terms, one of which may be series expanded. This step decomposes the L\'{e}vy measure into a sum of two measures and repeated application of this process produces the doubly indexed infinite series of measures. 
 
Again, we may avoid decomposition techniques particular to a given L\'{e}vy measure, and instead construct each $\nu_{k,h}$ directly from Theorem 1. We see that  each $\nu_{k,h}$ is of the form in \eqref{eq6} with 
\begin{equation*}
\eta(z)=\bigg(h,\frac{c(z)}{k+1}\bigg)\hspace{.2cm} \mbox{ and } \hspace{.2cm} A_0(z)=\frac{d\alpha(z)}{(k+1)^hh}F(z),
\end{equation*}
where $c(z)$ is a positive piecewise continuous function, and $F(z)$ is the cumulative distribution of $\alpha$.  Noting that the sufficient statistics for the gamma distribution are $T_1(x) = \ln(x)$ and $T_2(x) = x$,both of which have an infinitely differentiable inverse, condition 1 of Theorem 1 is satisfied. Again, similar to the case for the beta distribution,  the natural parameter space $\Xi$ of the gamma distribution consists of all $(\alpha,\beta)$ where $\alpha$,$\beta > 0$. From this it follows that conditions 2 and 3 of Theorem 1 are also satisfied.

\subsection{Construction of a completely random measure based on the Pareto distribution}

This section illustrates the construction of a completely random measure whose infinitesimal increments are \emph{Pareto distributed}. We will demonstrate the procedure using the change of variables formula from Theorem 1 explicitly. 

To begin consider the density function for a Pareto distributed random variable $p(u|\alpha) = \frac{\alpha u_m^{\alpha}}{u^{\alpha + 1}}$. In this form $\alpha$ $>0$ is the shape parameter, $u_m$ is the scale parameter, and the support of $p(u|\alpha)$ is $[u_m,\infty)$. Writing $p(u|\alpha)$ in canonical form yields $p(u|\alpha) = \exp\big(-(\alpha + 1)\ln(u)-(-\ln(\alpha)-\alpha \ln(u_m)\big)$. From this we see that there is one natural parameter, $-(\alpha + 1)$, and the corresponding sufficient statistic is $T(u)=\ln(u)$. Our goal is to choose an initial exponential family, apply the construction in Theorem 1, and then arrive,  after the change of variable,  at a form for the density of the L\'{e}vy measure in (3) which conforms to the density of a Pareto distribution. Thus, we consider the exponential family with sufficient statistics $T_1(x)=\mbox{ln}(x), T_2(x)=\mbox{ln}(\mbox{ln}(x))$, and natural parameters $ \eta_1 = -1, \eta_2 = -(\alpha + 1)$. We note that on $(1,+\infty)$ both sufficient statistics have a well defined inverse which is differentiable. Second, integration by parts on the integrand $x^{-1}(\mbox{ln}(x))^{-(\alpha +1)}$ shows that $(\eta_1, \eta_2)$ is in the natural parameter space of $p(x|\eta)$. Similarly, we also see that the natural parameter space is closed under contraction towards $0$. Hence we may apply the construction of Theorem 1 to produce a L\'{e}vy process $T$ so that the density of the corresponding L\'{e}vy measure, after performing the change of variables to express the L\'{e}vy measure as in equation (3) of the theorem, has the form

\begin{equation*} 
 \exp\big\{-(\alpha + 1)\ln(u) - u -(-\ln(\alpha)-\alpha \ln(u_m)\big\}e^u = \frac{\alpha u_m^{\alpha}}{u^{\alpha + 1}}
\end{equation*}
\noindent
which is a Pareto density.

More generally, taking $\alpha=\alpha(z)$, a positive piecewise continuous function on $(0,\infty)$, we see that conditions 2 and 3 of Theorem 1 are satisfied. Likewise, the sufficient statistic $T(x)$ chosen in the previous paragraph satisfies condition 1. Thus, there exists a L\'{e}vy process whose L\'{e}vy measure is given by $\alpha(z)u_m^{\alpha(z)}u^{-(\alpha(z)+1)}dA_0(z)du$ where $A_0(z)$ is any function satisfying the conditions of Theorem 1.

Note that, analogously to the previous two examples, we may consider the case where we wish to choose specific forms for $\alpha(z)$ for all $n=1,2,\ldots$, yielding $\{\alpha_n(z)\}_{n=1}^{\infty}$. We may produce $\nu_n(dz,du)=\alpha_n(z)u_m^{\alpha_n(z)}u^{-(\alpha_n(z)+1)}dA_0(z)du$ by applying the above construction to each $\alpha_n(z)$. Note that the $\nu_n(dz,du)$ for $n=1,2,\ldots$ are the L\'{e}vy measures acting as the intensity parameters for the countably many Poisson processes used to simulate the completely random measure whose corresponding L\'{e}vy measure is given by
\begin{equation*}
\sum_{n=1}^{\infty}\alpha_n(z)u_m^{\alpha_n(z)}u^{-(\alpha_n(z)+1)}dA_0(z)du.
\end{equation*}
Fixing an $\alpha(z)$ so that defining $\alpha_n(z)=n\alpha(z)$ and $dA_{0,n} = \frac{1}{n\alpha(z)}dz$ for $n=1,\ldots$, all $\alpha_n(z)$ and $A_{0,n}(z)$ satisfy the condition of Theorem 1, the above expression becomes
\begin{equation*}
\sum_{n=1}^{\infty}n\alpha(z)u_m^{n\alpha(z)}u^{-(\alpha(z)+1)}\frac{1}{n\alpha(z)}dzdu = \bigg(1 + \frac{u_m^{\alpha(z)}}{u^{-(\alpha(z)+1)}- u_m^{\alpha(z)}}\bigg)dzdu.
\end{equation*}

Thus, we have proved the existence of a completely random measure whose L\'{e}vy measure decomposition consists of the Pareto densities of our choice and whose composition of L\'{e}vy measures is obtained by the closed form expression
\begin{equation*}
dL_t(u) = \bigg\{\int_0^t\bigg(1 + \frac{u_m^{\alpha(z)}}{u^{-(\alpha(z)+1)}- u_m^{\alpha(z)}}\bigg)dz\bigg\}du.
\end{equation*}

Employing the L\'{e}vy measures $\{\nu_n\}_{n=1}^N$, where $N$ is a chosen level of truncation, simulation of this completely random measure may be achieved by Algorithm 1, analogous to the sampling algorithms based on decompositions of the beta and gamma processes.

\begin{algorithm}[t]
\begin{algorithmic}
\label{alg1}
\caption{Sampling algorithm for a CRM with L\'{e}vy measure Pareto density}
\REQUIRE $N$, $A_{0,n}(z)$, $\alpha(z)_nx_m^{\alpha_n(z)}u^{(-\alpha_n(z)+1)}$ for $n=1\ldots N$
\FOR{$n=1 \rightarrow N$}
\STATE $m_n \leftarrow$ Poisson$(\int A_{0,n}(dz))$
\FOR{$j=1 \rightarrow m_n$}
\STATE $z_{j,n} \xleftarrow{i.i.d.} \frac{A_{0,n}(dz)}{\int A_{0,n}(dz)}$
\STATE $u_{j,n} \xleftarrow{i.i.d.}  \alpha(z)_nu_m^{\alpha_n(z)}u^{(-\alpha_n(z)+1)}$
\ENDFOR
\ENDFOR
\RETURN $\bigcup_{n=1}^N \{(z_{j,n},u_{j,n})\}_{j=1}^{m_n}$
\end{algorithmic}
\end{algorithm}

\subsection{Conjugacy of a log-normal likelihood process and the gamma process}

In the remaining two sections of the examples we illustrate the simplicity of applying Theorem 2 to compute the posterior parameters for the nonparametric conjugate pair consisting of a log-normal likelihood process as the data generating process, and the gamma process as the prior. For ease of exposition we take the drift of our log-normal likelihood process to be zero, i.e. $\mu=0$. Note that this assumption can be easily dropped, although we do require $\mu$ to be known, as it is under this condition that the log-normal and gamma \emph{parametric} distributions are known to be conjugate. Thus, as we are assuming $\mu=0$, the parameter of interest is the variance of the likelihood process, namely $\sigma$.

The model assumption is that $\sigma$ is drawn from a gamma process constructed as per Theorem 1. Note, that is an important distinction for our prior under consideration, as we may construct a stochastic process $T(t)$ whose L\'{e}vy measure density is \emph{proper} and Gamma$(\alpha, \beta)$ distributed. This is in direct contrast to the well known gamma process appearing in section 5.2 and in \cite{Wang} for example, as the density of the process in that case is \emph{improper}. Hence we will call this process a \emph{generalized gamma} process and denote it by $\mbox{GenG}(\alpha(z),\beta(z)A_0(dz))$. 

We wish to employ Theorem 2 to compute the posterior parameters of the nonparametric process $T(t)|S_1,\ldots,S_n$, where $T(t) \sim \mbox{GenG}(\alpha(z),\beta(z)A_0(dz))$ and $S_1,\ldots,S_n$ are independent observations of the log-normal likelihood process with zero drift. Note that the form of the distribution of the infinitesimal increments of $T(t)$ is a consequence of the definition of the $T_{k,n,i}$ appearing in the construction of Theorem 1, and the scaling property of the gamma distribution.

In order to derive the posterior parameters of $T(t)|S_1,\ldots,S_n$ we must check the conditions of Theorem 2. To begin, the form of the function $\tau$ may be read off directly from the well known form for the posterior parameters in the parametric case. Thus, we have 
\begin{equation}
\tau(\alpha,\beta,\mu,x_1,\ldots,x_n)=\Big(\alpha + \frac{n}{2},\beta + \frac{\sum (\mbox{ln}(x_i)-\mu)^2}{2}\Big)
\end{equation}
\noindent
Next, conditions $1. - 3.$ follow directly from the fact that the prior process is constructed via Theorem 1 and thus from Theorem 2 we may conclude that 
\begin{equation}
T(t)|S_1,\ldots,S_n \sim \mbox{GenG}\Big(\alpha(z) + \frac{n}{2},\beta(z)A_{0,n,i}(dz) + \frac{\sum (\mbox{ln}(S_i)-\mu)^2}{2}\Big).
\end{equation}

\subsection{Conjugacy of a Pareto likelihood process and a Generalized gamma prior}

Consider the nonparametric conjugate pair consisting of a Pareto likelihood and a gamma prior. That a CRM exists with a density for its L\'{e}vy measure taking the form of a general gamma distribution is given by our main existence theorem. The construction of the Pareto likelihood is analogous to the construction of the Bernoulli likelihood from the beta process and was presented in the general case in section 4.2 of the present work. In this case we form the generalized gamma CRM from a concentration parameter $c$ and a base measure $G_0$, denoted by GenG($c$,$G_0$), by taking the parameters $a$ and $b$ in the gamma density to be $c$ and $cG_0$ respectively and invoking our main existence theorem.

Recalling the form of conjugacy in the parametric case for the Pareto likelihood and gamma prior, we find, employing Theorem 1, that given Pareto observations $X$ and our GenG($c$,$G_0$) prior, conjugacy in the nonparametric case is expressed as 
\begin{equation*}
 G|X_1,\ldots,X_n \sim \mbox{GenG}\left(c+n, \frac{c}{c+n}G_0 + \frac{1}{c+n}\sum \mbox{ln}\left(\frac{X_i}{x_m}\right)\right)
\end{equation*} 
 
\noindent
where $c$ is the concentration parameter, and $G_0$ is the initial measure. In fact, from the classical parametric form of conjugacy between the Pareto and gamma distributions, if $a$ and $b$ are the parameters governing the gamma distribution, then given $n$ observations of a Pareto random variable, the parameters of the posterior gamma distribution are given by 
\begin{equation*}
a_{post} = a + n,\hspace{1cm}b_{post} = b + \sum \mbox{ln}\left(\frac{X_i}{x_m}\right)
\end{equation*}

\noindent
Now, our definition of a GenG($c$,$G_0$) process is taken to be a stochastic process whose infinitesimal increments are  Gamma$(c,cG_0)$ distributed, i.e. $a = c$ and $b = cG_0$, and thus we see that the parameters of the posterior GenG($c$,$G_0$) process are given by 
\begin{equation*}
c+n = a + n
\end{equation*}

\noindent
and 
\small
\begin{equation*}
(c+n)\Big(\frac{c}{c+n}G_0 + \frac{1}{c+n}\sum \mbox{ln}\left(\frac{X_i}{x_m}\right)\Big)=cG_0 + \sum \mbox{ln}\left(\frac{X_i}{x_m}\right) = b + \sum \mbox{ln}\left(\frac{X_i}{x_m}\right).
\end{equation*}
\normalsize
\noindent
Hence, the form of the base measure of the posterior process is the weighted sum of the initial measure and an atomic component given by 
\begin{equation}
\frac{c}{c+n}G_0 + \frac{1}{c+n}\sum \mbox{ln}\left(\frac{X_i}{x_m}\right).
\end{equation}

\section{Conclusion and future work}

 Through our generalization of Hjort's construction via sufficient statistics we have addressed the problem of obtaining a general construction of prior distributions over infinite dimensional spaces possessing distributional properties amenable to conjugacy. In addition we have defined a theory of conjugacy for positive exponential families in infinite dimensional spaces analogous to the current theory of conjugacy for exponential families in a finite dimensional setting. This theory allows for explicit and efficient computation of nonparametric posterior parameters as a function of the given data and prior parameters. As such, a broader palette of nonparametric conjugate models, including both discrete and continuous likelihoods, are made available to the researcher.

Future work extending and complimenting the results in this paper will address a number of opportunities provided by the results achieved thus far. A detailed investigation of beta and gamma processes with improper vs. proper L\'{e}vy measure densities will aim to not only compare and contrast to properties of such processes, but will also attempt to reconcile in a rigorous manner the derivability of one from the other. 

Other certain areas of exploration will include applications of various nonparametric conjugate pairs, e.g. uniform and Pareto, to problem domains such as dynamic graphs, time series analysis, and topic modeling to name a few. In addition, the possibility of deriving efficient sampling algorithms from the construction and proof techniques of Theorems 1 and 2 provides an exciting avenue of future research.

\nocite{*}
\bibliography{Finn_Kulis}{}
\bibliographystyle{unsrt}

\end{document}